\documentclass[letterpaper]{article} % DO NOT CHANGE THIS
\PassOptionsToPackage{table,xcdraw}{xcolor}
\usepackage{aaai2026}  % DO NOT CHANGE THIS
\usepackage{times}  % DO NOT CHANGE THIS
\usepackage{helvet}  % DO NOT CHANGE THIS
\usepackage{courier}  % DO NOT CHANGE THIS
\usepackage[hyphens]{url}  % DO NOT CHANGE THIS
\usepackage{graphicx} % DO NOT CHANGE THIS
\usepackage{natbib}  % DO NOT CHANGE THIS AND DO NOT ADD ANY OPTIONS TO IT
\usepackage{caption} % DO NOT CHANGE THIS AND DO NOT ADD ANY OPTIONS TO IT
\usepackage{algorithm}

\usepackage{amsmath}    % For math environments
\usepackage{amssymb}    % For extra math symbols
\usepackage{bm}         % For bold math
\usepackage{booktabs}   % For better tables
\usepackage{multirow}   % For table rows
\usepackage{makecell}   % For formatting in table cells

\usepackage{my_packages}
\usepackage{newfloat}
\usepackage{listings}
\DeclareCaptionStyle{ruled}{labelfont=normalfont,labelsep=colon,strut=off} % DO NOT CHANGE THIS
\floatstyle{ruled}
\newfloat{listing}{tb}{lst}{}
\floatname{listing}{Listing}
\title{Robust Behavior Cloning Via Global Lipschitz Regularization}
\author{
    Shili Wu\textsuperscript{\rm 1}, 
    Yizhao Jin\textsuperscript{\rm 2}, 
    Puhua Niu\textsuperscript{\rm 1}, 
    Aniruddha Datta\textsuperscript{\rm 1}, 
    Sean B. Andersson\textsuperscript{\rm 3}\textsuperscript{\rm 4}
}
\affiliations{
    \textsuperscript{\rm 1}Dept. of Electrical and Computer Engineering, Texas A\&M University, College Station, TX 77843, USA\\
    \textsuperscript{\rm 2}Game AI Research Group, Queen Mary University of London, London E1 4NS, UK\\
    \textsuperscript{\rm 3}Dept. of Mechanical Engineering, Boston University, Boston, MA 02215, USA\\
    \textsuperscript{\rm 4}Division of Systems Engineering, Boston University, Boston, MA 02215, USA\\
    shiliwu@tamu.edu, acw596@qmul.ac.uk, niupuhua.123@tamu.edu, a-datta@tamu.edu, sanderss@bu.edu
}

\usepackage{bibentry}
\begin{document}

\maketitle

\begin{abstract}
Behavior Cloning (BC) is an effective imitation learning technique and has even been adopted in some safety-critical domains such as autonomous vehicles. BC trains a policy to mimic the behavior of an expert by using a dataset composed of only state-action pairs demonstrated by the expert, without any additional interaction with the environment. However, During deployment, the policy observations may contain measurement errors or adversarial disturbances. Since the observations may deviate from the true states, they can mislead the agent into making sub-optimal actions. In this work, we use a global Lipschitz regularization approach to enhance the robustness of the learned policy network. % by certifying the output at every step. 
   We then show that the resulting global Lipschitz property provides a robustness certificate to the policy with respect to different bounded norm perturbations. Then, we propose a way to construct a Lipschitz neural network that ensures the policy robustness. We empirically validate our theory across various environments in \textit{Gymnasium}.
\end{abstract}

\section{Introduction}
\label{sec:intro}
Deep Reinforcement Learning (DRL) methods often require extensive interaction with the system for effective training, as well as carefully crafted reward structures to effectively guide the search for an optimal policy. Such training can be especially challenging in complex environments~\citep{tai2016survey}. Motivated by this, Behavior Cloning (BC)~\citep{ross2010efficient, ross2011reduction} was developed as a means of finding an optimal policy by learning from expert demonstrations. BC is widely used to train policies in fields such as autonomous driving~\citep{codevilla2019exploring} and robotics~\citep{jang2022bc}. As with reinforcement learning, the policy is trained offline. During deployment, the observations may be corrupted by noise or disturbances and it is beneficial for the learned policy to be robust against such uncertainty.

Approaches to policy robustness in the Reinforcement Learning (RL) literature can be divided into two main themes: robustness to model uncertainty and robustness to state observation. The former explores the performance of a policy when the conditions during training and testing differ, particularly with respect to the details of the transition model, using the framework of a Robust Markov Decision Process (RMDP)~\citep{nilim2005robust}. For the use of BC within such settings, see~\citep{panaganti2023distributionally}. The latter theme addresses robustness against observational uncertainty, driven by sensor noise or error. This remains relatively unexplored when BC is used to train a policy and is the core challenge we seek to address in this work.

Although prior work on the robustness of BC to observational noise is limited, there is existing literature on robustness against state observation error in the context of RL where the problem is often formulated as a State Adversarial Markov Decision Process (SA-MDP) \citep{zhang2020robust}. To find a robust policy, %this method employs a smoothing regulator with objectives similar to those described in \cite{shen2020deep}. In that work, 
an adversary is first trained to identify the error on the state observation (from within a bounded set) that causes the most significant change in the policy's action output from the noise-free response. Robustness is achieved by seeking a policy that minimizes the discrepancy between the action resulting from the noise-free observation and the noisy one. Intuitively, this approach can be understood as trying to maintain similarity in the actions applied from nearby states so as to yield consistent performance~\citep{shen2020deep}. However, as discussed in \citep{zhang2020robust}, this approach does not provide a \textbf{certificate of robustness}, which is essential in some safety-critical applications. That is, it does not guarantee a bound on the drop in the reward under an arbitrary perturbation (from that bounded set). The idea of a certificate of robustness is further discussed in ~\cite{kumar2021policy}. In that work, Gaussian noise is injected into the training data and it is shown theoretically that after smoothing the policy, the mean cumulative reward would not drop below a certain threshold in the face $l_2$-bounded observation noise. However, these works only consider stochastic policies. 

Until recently, the relationship between the robustness of a policy and its Lipschitz continuity has been established in several works, either theoretically \citep{bukharin2023robust} or experimentally \citep{nie2023improve, pmlr-v202-song23b}. Specifically, \cite{bukharin2023robust} demonstrates that the Lipschitz continuity of the value function can be linked to the Lipschitz continuity of the policy and employs adversarial regularization to smooth the policy. However, this approach does not provide a certificate of robustness due to the nature of adversarial regularization. On the other hand, \cite{nie2023improve, pmlr-v202-song23b, zhang2022rethinking} proposed to use \textbf{Lipschitz networks}, which are specially designed networks whose structure allows control over how much its output changes relative to its input. However, the detailed relationship between the Lipschitzness of a policy and the Lipschitzness of the neural networks has not been discussed as yet.

\textbf{Technical Contributions} In this work, we extend the analysis of the SA-MDP framework to incorporate deterministic policies, because deterministic policies are a commonly used form both in RL and BC. Then, we establish the relationship between the Lipschitz continuity of a policy and the robust certificate, which enables us to quantify the performance drop of the policy when it faces the worst uncertainties. Furthermore, we demonstrate that controlling the Lipschitz continuity of a neural network inherently ensures the Lipschitz continuity of the corresponding policy, regardless of whether the policy is deterministic or stochastic. Finally, we apply these findings to behavior cloning, showcasing that our approach produces robust policies in practice.

\section{Preliminaries}
\label{sec:prem}
\paragraph{Notation} We use calligraphic letters to denote sets, %such as $\mathcal{S}$ and $\mathcal{A}$. 
For any vector $x$, The $p$-norm is denoted by $\Norm{\cdot}$, and the Lipschitz semi-norm is expressed as $\Norm{\cdot}_L$.
%The notation $\mathrm{D}(\cdot, \cdot)$ refers to a specific distance metric (e.g., $\mathrm{D}_{TV}(\cdot, \cdot)$ denotes the total variation distance), while lowercase $d$ with a subscript specifies the distance function within sets. 

\paragraph{Markov Decision Process} We model the sequential decision making problem as a Markov Decision Process (MDP) defined as a tuple $\mathcal{M} = (\mathcal{S}, \mathcal{A}, \bP, R, \gamma, \nu_0)$, where $\mathcal{S}$ is a bounded metric state space (that is, for a given distance metric $d_\bS(\cdot,\cdot): \mathcal{S} \times \mathcal{S} \rightarrow \mathbb{R}^+$, we have that $d_\bS(s, s') \leq C$  $\forall s, s' \in \bS$ for some $C \in \mathbb{R}$), $\mathcal{A}$ is the action space, $R(\cdot, \cdot): \mathcal{S} \times \mathcal{A} \rightarrow \mathbb{R}^+$ is the reward function (assumed to be bounded on the interval $[0, R_{\max}]$), $\bP: \mathcal{S} \times \mathcal{A} \rightarrow \mathcal{P}(\mathcal{S})$ is the transition probability function, $\gamma \in (0, 1)$ is the discount factor, and $\nu_0$ is the initial state distribution over $\mathcal{S}$. The transition probability function takes the form $\mathcal{P} = Pr(s_{t+1} = s' | s_t = s, a_t = a)$, and $R_t = R(s_t, a_t)$ denotes the reward received at discrete time $t$.

A  policy $\pi: \mathcal{S} \rightarrow \mathcal{A}$ specifies the action to take as a function of the state. Note that we consider only stationary policies. Execution of a policy determines a $\gamma$-discounted visitation distribution, defined as $\rho^\pi(s) = (1 - \gamma) \prod_{t=0}^{+\infty} \gamma^t \mathcal{P}(s_t|\pi, \nu_0)$. The state value function (or $V$-function) is expressed as the expected discounted sum of rewards obtained under a policy $\pi$ that starts from a given state $s \in \bS$ starting from time $t$, shown as follow:
\begin{align*}
    V^\pi(s) &:= \E_\pi \left[\sum_{k=0}^{+\infty} \gamma^k R_{t+k+1} | s_t = s\right].
\end{align*}
% It is easy to see that $\forall s, s' \in \bS, L_\pi d_\bS(s, s') \leq 1$. This is because $TV(\cdot, \cdot)$-distance is always upper bounded by $1$.
\paragraph{SA-MDP} Following the work in \cite{zhang2020robust}, we restrict the perturbations of the adversary to a bounded set of states and then introduce the definition of a deterministic adversary $\mu$.

\begin{definition} (Adversary Perturbation Set) We define a set $B_\epsilon(s)$ which contains all allowed perturbations of the adversary. Formally, $\mu(s) \in B_\epsilon(s)$ where $B_\epsilon(s) = \{ s' \in \bS | d_\bS(s', s) \leq \epsilon \}$.
\end{definition}

\begin{definition} (Stationary, Deterministic and Markovian Adversary) An adversary $\mu(s)$ is a deterministic function $\mu: \mathcal{S} \rightarrow \mathcal{S}$ which only depends on the current state $s$ and does not change over time.
\end{definition}

%\cite{kumar2021policy} shows that a deterministic adversary can be the optimal adversary among all adaptive adversaries. Then, by 

The adversary $\mu$ seeks to optimize the following objective:
\begin{align*}
    \min_\mu \E_\pi\big[\sum_{k=0}^\infty \gamma^k R_{t+k+1}\big], \quad a_t \sim \pi(\cdot|\mu(s)) \text{ and } \mu(s) \in B_\epsilon(s). 
\end{align*}
Note that here $\pi(\cdot, \mu(s))$ is the policy using the \textit{perturbed} state, not the true one.
%\textbf{Robust MDP(RMDP)} We use the RMDP framework where the $\bP$ is defined to be an uncertainty set as
%\begin{align*}
%    \bP &= \otimes_{(s,a) \in \bS \times \bA} \bP_{s,a}\\
%    \textit{s.t} \quad\bP_{s,a} &= {\bP_{s,a} \in \Delta(\bS) : D(P_{s,a}, P_{s, a}^0)\leq \rho_r'}
%\end{align*}
%where $\bP^0 = (P_{(s, a)}^0(s,a) \in \bS \times \bA)$ is the simulator model, $D(\cdot, \cdot)$ is a distance measure between two probability distributions and $\rho_r$ is the radius of the uncertainty set indicating the level of robustness. We assume the real-world model belongs to this uncertainty set $\bP$. We restrict to the total variation distance $D_{TV}$ for the measure $D$ in this paper and leave other types of measures for furture work.
%Then the expected return $J(\pi)$ of a policy $\pi$ can be expressed as $J(\pi) = \E_\pi[\sum_{t=0}^\infty \gamma^t R_t] = \E_{s_0 \sim \nu_0} V^\pi(s_0)$ and clearly we have $J(\pi) \leq \max_{s} V^\pi(s)$. 

\cite{zhang2020robust} showed that a given policy $\pi$ can be merged with a fixed adversary $\mu$ into a single decomposed policy $\pi \circ \mu$. The value function for such a decomposed policy $\pi \circ \mu$ is given by
\begin{align*}
    V^{\pi \circ \mu} = \E_{\pi \circ \mu}[\sum_{k=0}^\infty \gamma^k R_{t+k+1} | s_t = s].
\end{align*}    
The details can be found in \cite{zhang2020robust} (see also Appendix A) %\ref{sec: appendix prelim}.
%The state value function (or $V$-function) and state-action function (or $Q$-function) are expressed as the expected discounted sum of rewards obtained under a policy $\pi$. The difference is that $V$ starts from a given state $s \in \bS$ and $Q$ start from a state-action pair $(s, a)$, where $a \in \bA$.
%\begin{align*}
%    V^\pi(s) &:= \E_\pi [\sum_{t=0}^{+\infty} \gamma^t R_t | s_0 = s] \\
%    Q^\pi(s, a) &:= \E_\pi [\sum_{t=0}^{+\infty} \gamma^t R_t | s_0 = s, a_0 = a]
%\end{align*}
% We use $(V^\pi_\bP, Q^\pi_\bP)$ to denote the state value function and state-action value function when executing a policy $\pi$ under a transition probability model $\bP$.
\paragraph{Behavior Cloning} Unlike RL, which aims to discover a policy that maximizes cumulative rewards through interactions with the environment, BC seeks to find a policy, $\pi_I$, that can mimic the behavior of another policy, namely the expert policy $\pi_E$. This process begins with a pre-collected dataset generated by $\pi_E$, consisting of state-action pairs $(s, a)$. This dataset serves as the training data for BC, where the model learns to reproduce the expert's decisions without knowing the rewards. We follow a common way to perform behavior cloning: minimizing a given distance function between the expert's action and the imitator selection. We are thus interested in solving
\begin{align}
    \label{prob: vanila bc}
    \pi_I = \argmin_{\pi} \E_{s \sim \rho^E} [\mathrm{D}(\pi_{E}(s), \pi_{I}(s))],
\end{align}

where $\rho^E$ is the state distribution in the data set collected by the expert policy, $\pi_E$ is the expert policy that we are attempting to mimic, $\pi_I$ is the imitator policy, and $D$ is a distance function depending on the action space type. In this work, we use the $l_2$ distance for deterministic policies and the Total Variation (TV) for stochastic policies with a discrete action space.

\paragraph{Lipschitz Policy} 

%The smoothness of a policy is intrinsically linked to its Lipschitz constant, $L_\pi$, which varies depending on the metric used. As in BC, the choice of distance metric depends on the policy type and we adopt the Euclidean norm ($\Norm{\cdot}_2$) for deterministic polices and the TV distance($\mathrm{D}_{TV}$) for stochastic ones. To provide a consistent framework for discussing these metrics across different types of action spaces, we denote the distance metric as $\mathrm{D}(\cdot, \cdot)$.

We can describe the \textit{local} Lipschitz property of a policy using a function $L_\pi^{\epsilon}(s)$ that depends on the state $s$ and its $\epsilon$-neighborhood, which is formally defined as follows:
\begin{definition}
    The \textbf{local} Lipschitz constant of policy $\pi$ in a neighborhood $\epsilon$ of a given state $s$ is $L_\pi^{\epsilon}(s) = \sup_{s'\in B_\epsilon (s), s \neq s'} \frac{\mathrm{D}(\pi(s), \pi(s'))}{d_\bS(s, s')}$.
\end{definition}
%\begin{remark}
%    In this work, the distance metric for $d_\bS(s', s)$ is taken to be a given $p$-norm since the states are usually represented as elements of $\mathbb{R}^n$ in most neural network approaches. To provide a consistent framework for discussing these metrics across different types of action spaces, we denote the distance metric as $\mathrm{D}(\cdot, \cdot)$. Specifically, we adopt the Euclidean norm ($\Norm{\cdot}_2$) for deterministic actions and the TV distance($\mathrm{D}_{TV}$) for stochastic ones. 
%\end{remark}

The \textbf{global} Lipschitz constant $L_\pi$ is an upper bound of the local Lipschitz constant such that $L_\pi^{\epsilon}(s) \leq L_\pi$.

\begin{definition}
    A policy $\pi$ is said to be $L_\pi$-Lipschitz Continuous~($L_{\pi}$-LC) if its Lipschitz constant $L_\pi \leq \infty$, where
    $L_\pi = \sup_{(s, s') \in \bS^2, s \neq s'} \frac{\mathrm{D}(\pi(s), \pi(s'))}{d_\bS(s, s')}$
\end{definition}

Following a standard assumption \citep{bukharin2023robust}, we assume the environment is $(L_p, L_r)$-Lipschitz, which is formally defined as follows:
\begin{definition} (Lipschitz MDP)
    A MDP is $(L_p, L_r)$-Lipschitz if the rewards function and transition probability satisfy
    \begin{align*}
        |R(s, a) - R(s', a')| \leq L_r (\mathrm{D}_\bA (a, a') + d_\bS (s, s')),\\
        \mathrm{D}_{TV} (p(\cdot | s, a), p(\cdot | s', a')) \leq L_\bP (\mathrm{D}_\bA (a, a') + d_\bS (s, s')).
    \end{align*}
\end{definition}

\section{Behavior Cloning with a Lipschitz Policy Neural Network}

%While it is rather easy to design a $L_\pi$-LC policy network for many parameterized distribution as they possess a close form in $TV$-distance (we give an example in Appendix), Although there may be alternative approaches, the one used by us produces a network comprised of two components: an $M$-layer that is fully connected and an output layer with a $softmax(\cdot)$ activation, which is a common configuration prevalent in the reinforcement learning literature.

%\subsection{Lipschitz Neural Network}
%We estimate the lipschitz constant $L_\pi^{FC}$ of the fully connected neural network by the following theorem
\subsection{Certified Robustness from Lipschitz Constant}
%We view robustness as a property that is not only applicable to the optimal policy but is shared among all policies $\pi$ under the same MDP. 
In this section, we first present a metric to assess the level of robustness for a given policy $\pi$ and then we show how the Lipschitz constant of a policy $\pi$ is related to this robustness certificate. We then use these results to inform the choice of a new objective function for training based on BC. We start with defining the robustness certificate.
\begin{definition}
    The robustness certificate $\Theta(\pi)$ of a policy $\pi$ is defined to be  
    \begin{align*}
    \Theta(\pi) = \max_s [V^{\pi}(s) - \min_{\mu} V^{\pi \circ \mu}(s)].
\end{align*}
\end{definition}
\begin{remark} $\Theta$ can be understood as the worst-case performance drop when $\pi$ when the state observation is subject to perturbations. This objective can be translated to the objective of maintaining the expected cumulative reward above a certain threshold in previous work (e.g \cite{kumar2022certifying}).
\end{remark} %Unlike the previous certifiable robustness objective that finds a policy $\pi$ such that the total reward in the presence of a norm-bounded adversary is guaranteed to remain above a certain threshold \cite{kumar2021policy}, this objective focuses on the performance drop when a policy facing the perturbation. We implemented this modification not only because the rewards for both the expert and the imitator are often unknown, but also because it aligns with the intuitive concept of robustness. Under the previous definition, a policy $\pi$ with a high $V^\pi$ could be considered as "robust" in the sense that it can withstand greater losses. \sba{I commented out this piece as I feel that kind of discussion should be in the introduction, not here. I also think you made that point in the introduction.}

%Our objective is then to find the optimal policy $\pi$ such that the performance drop of that policy $\pi$ will be bounded by an $\underline{R}$ under the worst-case noise scenario. i.e
%\begin{align*}
%    &\max_{\pi} \E_\pi [\sum_{t=0}^\infty \gamma^{k} R_{k+t+1}] \quad \textit{s.t} \quad \Theta(\pi) \leq \underline{R}
%\end{align*}
We first discuss the robustness level of an $L_\pi$-LC policy $\pi$ when it encounters a bounded adversary $\mu$ or, equivalently, the performance advantage of policy $\pi$ with respect to the merged policy $\pi \circ \mu$. Using the Performance Difference Lemma (PDL) \citep{Kakade2002ApproximatelyOA}, we can prove the following theorem. The technical details of this proof are presented in Appendix A. %as the proof of Theorem \ref{proof: Lpi robust certificate}.
\begin{theorem}\label{thm: performance drop}
    Consider an infinite horizon ($L_P, L_r$)-LC SA-MDP $\bM = \{ \bS, \bA, R, \bP, \gamma, B_\epsilon \}$ with a bounded reward function, $0 \leq R \leq R_{\max}$, and policy $\pi$ with a local Lipschitz function $L_\pi^\epsilon$. Then, for small enough $\epsilon$ and any fixed bounded adversary $\mu: \bS \rightarrow \bS$, we have that
    \begin{align*}
        \Theta(\pi) %= \max_s [V^\pi(s) - V^{\pi \circ \mu}(s)] \sba{You defined this above- no need to repeast}
        \leq \alpha \mathrm{D}(\pi(\cdot, s), \pi(\cdot, \hat{s})) \leq \alpha L_\pi^\epsilon(s) \epsilon, \quad \text{for any } \hat{s} \in B_\epsilon (s), \forall s \in \bS.
    \end{align*}
    where $\alpha = \frac{R_{\max}}{(1-\gamma)^2}$ when policy $\pi$ is stochastic and $\alpha = \frac{1}{(1-\gamma)} (L_r + \frac{\gamma R_{max}}{1 - \gamma}L_P)$ when policy $\pi$ is deterministic. 
\end{theorem}
\begin{remark}
    This theorem extends Theorem 5 in \cite{zhang2020robust}. Specifically, we demonstrate that the original theorem can be generalized to handle deterministic policies, as the use of the TV metric for measuring policy differences in the original formulation renders it ineffective (taking values of either 1 or 0) for deterministic cases. Furthermore, our theorem introduces a novel connection to the Lipschitz property of the policy.
\end{remark}

\begin{remark}
    The smoothing regularizer in various works (see, e.g. \cite{zhang2020robust, shen2020deep, pmlr-v202-song23b, nie2023improve}) can be regarded as finding a policy $\pi$ (parametrized by $\theta$) with a lower upper bound of this local Lipschitz function $L_\pi^\epsilon(s)$ 
    %enhances robustness of the policy $\pi$ (parametrized by $\theta$) by learning a smooth regularizer $\mathcal{R}$, defined by
\end{remark}
However, deriving the robustness certificate from $L_\pi^\epsilon(s)$ is computationally intractable since it needs to compute the magnitude of the gradient at each of the system states. Alternatively, we can use the global Lipschitz constant $L_\pi$ to derive such a certificate. To do this, we enforce a constraint on the learned policy in terms of a given bound on the global Lipschitz constant and replace \eqref{prob: vanila bc} with
%Therefore, to ensure the robustness of the learned policy using BC, we enforce a constraint on the learned policy in terms of a given bound on the global Lipschitz constant. That is, w
\begin{align*}
    \pi_I = \arg \min_{\pi}\E_{s \sim d^E} [\mathrm{D}(\pi_{E}(s), \pi_{I}(s))] \\
    \quad \textit{subj. to} \quad \Norm{\pi_{I}}_L \leq L_\pi
\end{align*}
where $\Norm{\pi_I}_L$ is shorthand for the Lipschitz constant of the policy. It is worth noting, however, that this conservativeness means that a policy with a large Lipschitz constant does not necessarily mean that it is non-robust. Another downside of using such a hard constraint is that the imitator policy may experience significant interpolation errors. These errors arise when the policy fails to accurately mimic the expert policy's behavior, leading to a decline in performance. Consequently, the worst-case expected reward difference between the expert policy and the imitator policy under the worst-case perturbation is the sum of the interpolation error and the robust certificate. This trade-off of performance and robustness is the same as that of when using adversarial training \citep{zhang2019theoreticallyprincipledtradeoffrobustness}. Despite this, the ability to produce a certificate of robustness can be quite valuable and, as shown in Sec.~\ref{sec:experiments}, this does not seem to affect performance significantly in practice.

%\sba{makes sense- but this is a fairly negative statement with which to finish the description of your idea. Is there a postiive you can finish with?} \sw{I can only say we can probably estimate the worst case from both the interpolation error and the global lipschitz constant? otherwise there is no way, due to that BC is reward agnostic. Also, many previous work suffers from this problem as well.}

%A similar bound $2[1 + \frac{\gamma}{(1-\gamma)^2}]R_{\max}L_\pi \epsilon$ can be derived from Theorem 5 in \cite{zhang2020robust}, but ours is tighter. 
% Theorem \ref{thm: performance drop} shows that deploying a smaller $L_\pi$-LC policy can offer the same robustness certificate while accommodating a larger observation noise $\epsilon$.
\subsection{Behavior Cloning with a Lipschitz Neural Network}
%The weights of the neural network can be regularized with respect to different $l_p$ norms. Here, we mainly consider the $l_\infty$-norm of the neural network and show that it provides a robustness certificate %the $L_\pi$ the policy. 
%We also show regulating the $l_\infty$ can provide a robust certificate for 
%for a bounded perturbation under different $p$-norms.
In this section, we will first show that regulating the $l_\infty$-norm of the neural network can be sufficient to provide robust certificates for various forms of policies. Then, we present a procedure on how to ensure the Lipschitz constant of the network reaches a desired level, $L_\pi$. As the policy may have different forms, we show how $l_\infty$-norm of the neural network is related to the $L_\pi$ respectively.

\paragraph{Policy and Neural Network Lipschitz Relationship} In the case of a deterministic policy (where the action is directly output by the neural network), obtaining the Lipschitz constant is straightforward as shown in the theorem below.
\begin{theorem} [Lipschitzness of a Deterministic Policy] \label{thm: robust certificate of deterministic policy}
    Consider an $M$-layer fully connected neural network $f: \mathbb{R}^n \rightarrow \mathbb{R}^m$ %$f: (\mathbb{R}^n, \Norm{\cdot}_\infty) \mapsto (R^m, \Norm{\cdot}_\infty)$, \sba{I presume you mean the vector space with the given norm? Notation is a bit non-standard} \sw{it means maps from a Rn equipped with the inifty norm as distance metric to Rm equipped with the inifty norm. That's what I thought- but I also dont' think you neec to specify the norm there, so I omotted it.}
    where the $i^{th}$ layer has weights $W_i$, bias $b_i$, and a 1-Lipschitz activation function (such as a ReLU). Then $\forall x, y \in \mathbb{R}^n$
    \begin{align*}
        \Norm{f(x) - f(y)}_{2} \leq m\Norm{f(x) - f(y)}_\infty \leq m L^{FC}_\pi \Norm{x-y}_p
    \end{align*}
where $L^{FC}_\pi = \prod_{i=0}^M \Norm{W_i}_\infty$, the $\infty$-norm for the matrix denotes the induced matrix norm, and $m$ is the dimension of the output of the network.
\end{theorem}

\begin{remark}
    The proof of this theorem is provided in Appendix A. %\ref{sec: lips bound on deterministic}.
    The first inequality indicates that controlling the infinity-induced matrix norm $\Norm{W_i}_\infty$ is sufficient to regulate the 2-norm of the output. Then, the second inequality allows us to infer the Lipschitz constant $L_\pi$ when using this $f$ as the policy $\pi$ because the input dimension $m$ is always a fixed and known constant. Furthermore, combining \ref{thm: robust certificate of deterministic policy} with Thm. \ref{thm: performance drop}, we can get the robust certification $\Phi(\pi)$.
\end{remark}
\begin{remark}
    Note that $L^{FC}_\pi$ is a loose upper bound. Deriving an exact or tight upper bound is computationally expensive. However, the loose upper bound is sufficient for obtaining a robustness certificate.
\end{remark}
In the case that a stochastic policy operates over a discrete action space, it is common to use the $softmax$ function to output the probability of selecting each action. Then, $L_\pi$ can be derived from the following.
\begin{theorem} [Lipschitzness of Categorical Policy]\label{theorem: Lpi of a neural network}Consider a neural network $f:\mathbb{R}^n \rightarrow \Delta(\mathbb{R}^m$), where $\Delta$ is the induced probability distribution on the action space with two components: 
    an $L^{FC}_\pi$-LC fully connected network $f^{FC}: \mathbb{R}^n \rightarrow \mathbb{R}^m$, and an $softmax(\cdot)$ output layer. Then $\forall x, y \in \mathbb{R}^n$
    \begin{align*}
        \mathrm{D}_{TV}\left(f(x), f(y)\right) \leq \frac{m^2}{2} L^{FC}_\pi \Norm{x - y}_p.
    \end{align*}
\end{theorem}
\begin{remark}
    The detailed proof can be found in Appendix A. %\ref{sec: lips bound on softmax}. 
    Similar to Thm. \ref{thm: robust certificate of deterministic policy}, the constant $L_\pi$ for the discrete action policy can also be inferred from the value of $L_{\pi}^{FC}$ as the $m$ remains a fixed and known constant when using this $f$ as our $\pi$, and we can get the corresponding robust certificate $\Phi(\pi)$ by combining \ref{theorem: Lpi of a neural network} with Thm. \ref{thm: performance drop}.
\end{remark}

In the context of a stochastic policy over a continuous action space, it is also common to employ a Gaussian distribution where the neural network outputs the mean and standard deviation. Because the total variance between two Gaussian distributions can be upper bounded by their KL divergence, which can be easily computed by their mean and standard deviations ~\citep{zhang2020robust}, the policy Lipschitzness can also be computed from the Lipschitzness of the neural network (see Appendix A
%\ref{sec: lip bound on general stochastic} 
for detail). Note that such a policy is rarely used with BC and we only include it for completeness.

%As the KL divergence of two Gaussian distributions can be inferred from their means and standard deviations and can provide an upper bound for the total variation distance between them~\citep{zhang2020robust} {\color{red} TODO: be more direct}, Theorem \ref{theorem: Lpi of a neural network} is easily extended to this setting (see Appendix \ref{sec: lip bound on general stochastic}). Note that such a policy is rarely used with BC and we only include it for completeness.

\paragraph{Regulating Method} The above results show that constructing a Lipschitz network with constant $L_{\pi}^{FC}$ in turn establishes the Lipschitz bound for the policy network, $L_{\pi}$. Thus, regulating the $\infty$-induced matrix norm for the network weights gives a certificate of robustness against bounded state perturbations. 
%(with respect to different $p$-norms). \sba{Shili- please check that this statement is correct.} \sw{Yes, but the constant would be different due to the switch of the $p$-norm}.\sba{Ok- I took out the reference to the p-norms here because I think it muddies the situation and isn't particularly relevant to this statement.}
%Until here, we show that bounding the $L_\pi^{FC}$ is the same as bounding the $L_\pi$ in most common practice, \sba{what do you mean 'in most common practice'. Be specific to what you've done} \sw{These are the only policy representations as far as I know, but I don't know if there is another forms of policy}and regulating the $\infty$-induced matrix norm can give a certificate of robustness again bounded state perturbations (in different norms). 
There are multiple ways to construct the $L_\pi$-LC neural network (which we hereafter refer to as LipsNet) (such as \cite{pmlr-v202-song23b},\cite{liu2022learning}). Here we follow the approach in \cite{liu2022learning} which uses a Weight Normalization (WN) technique to construct the network since it does not need to compute the Jacobian matrix at each state $s$ and as a result is more computationally efficient than other methods. The WN technique in \cite{liu2022learning} regulates the neural network using an auxiliary loss function $\mathcal{L}$. It augments each layer of a fully connected neural network with a trainable Lipschitz bound $c_i$ on the infinity norm of the $W_i$ as shown in Fig. \ref{fig:LipsNet-overview}. The output of each layer is then calculated as
\begin{align*}
    y_i = \sigma_i(\hat{W}_i x_i + b_i), \quad \hat{W}_i = W_i\frac{\textit{softplus} (c_i)}{\Norm{W_i}_\infty} 
\end{align*}
where $\textit{softplus}(x) = \ln (1+\exp(x))$ is implemented to prevent $\Norm{\hat{W}_i}_\infty$ from going to $0$. The auxiliary loss function is defined by $\mathcal{L}(\theta) = \lambda \prod_{i=1}^M \textit{softplus} (c_i)$. 

In contrast to the result in \cite{liu2022learning} where the focus is on finding the smoothest function (i.e., minimizing the Lipschitz constant), we are interested in setting the Lipschitz constant of the policy network to $L_\pi^{FC}$ while maintaining the expressive capabilities of the policy network. Therefore, we establish a lower limit for $c_i$ by using a $\max(\cdot, 0)$ function. The modified loss function becomes
\begin{align*}
    \mathcal{L}(\theta) = \lambda \max\left( \prod_{i=1}^M \textit{softplus} (c_i) - L_\pi^{FC}, 0\right).
\end{align*}
The intuition behind the modified loss function is that if $c_i$ falls below this lower threshold, it will result in a loss of $0$ and will not further regulate the neural network.

\begin{figure}[t]
\centering
\includegraphics[width=0.95\columnwidth]{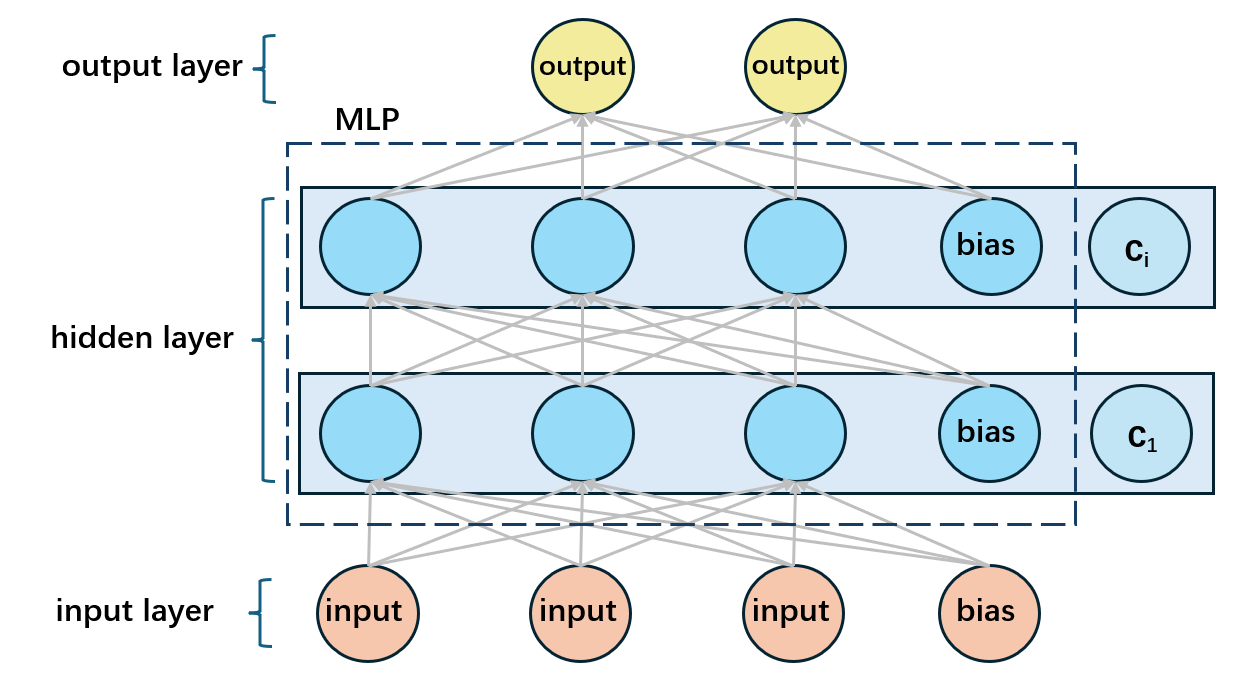}
\caption{LipsNet network configuration. Each MLP layer is appended with trainable parameters $c_i$ for weight normalization.}
\label{fig:LipsNet-overview}
\end{figure}

\begin{figure*}[t]
\centering

\begin{minipage}[b]{0.24\textwidth}
    \centering
    \includegraphics[width=\textwidth]{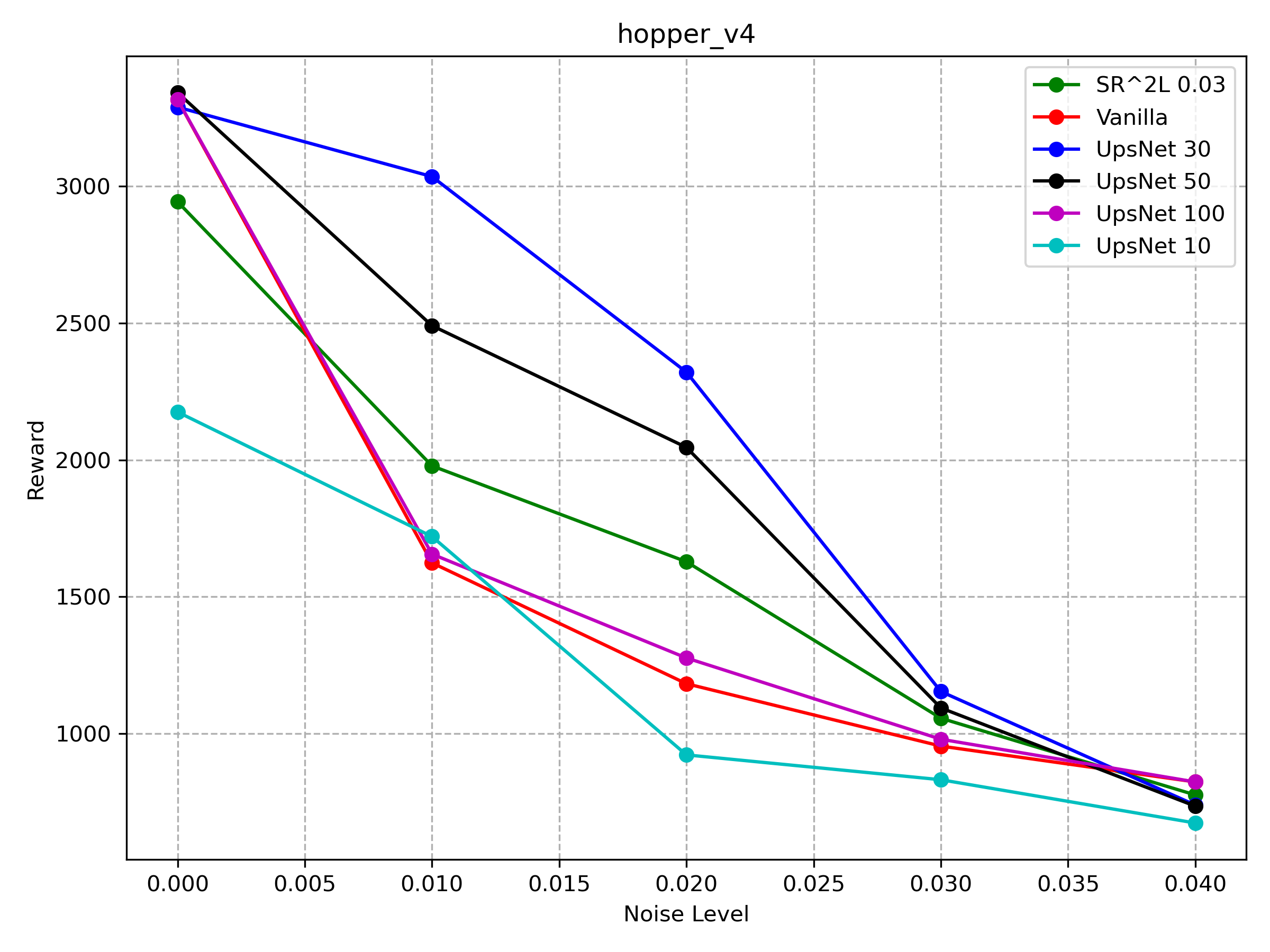}
    \caption*{(a) Hopper-v4}
\end{minipage}
\hfill
\begin{minipage}[b]{0.24\textwidth}
    \centering
    \includegraphics[width=\textwidth]{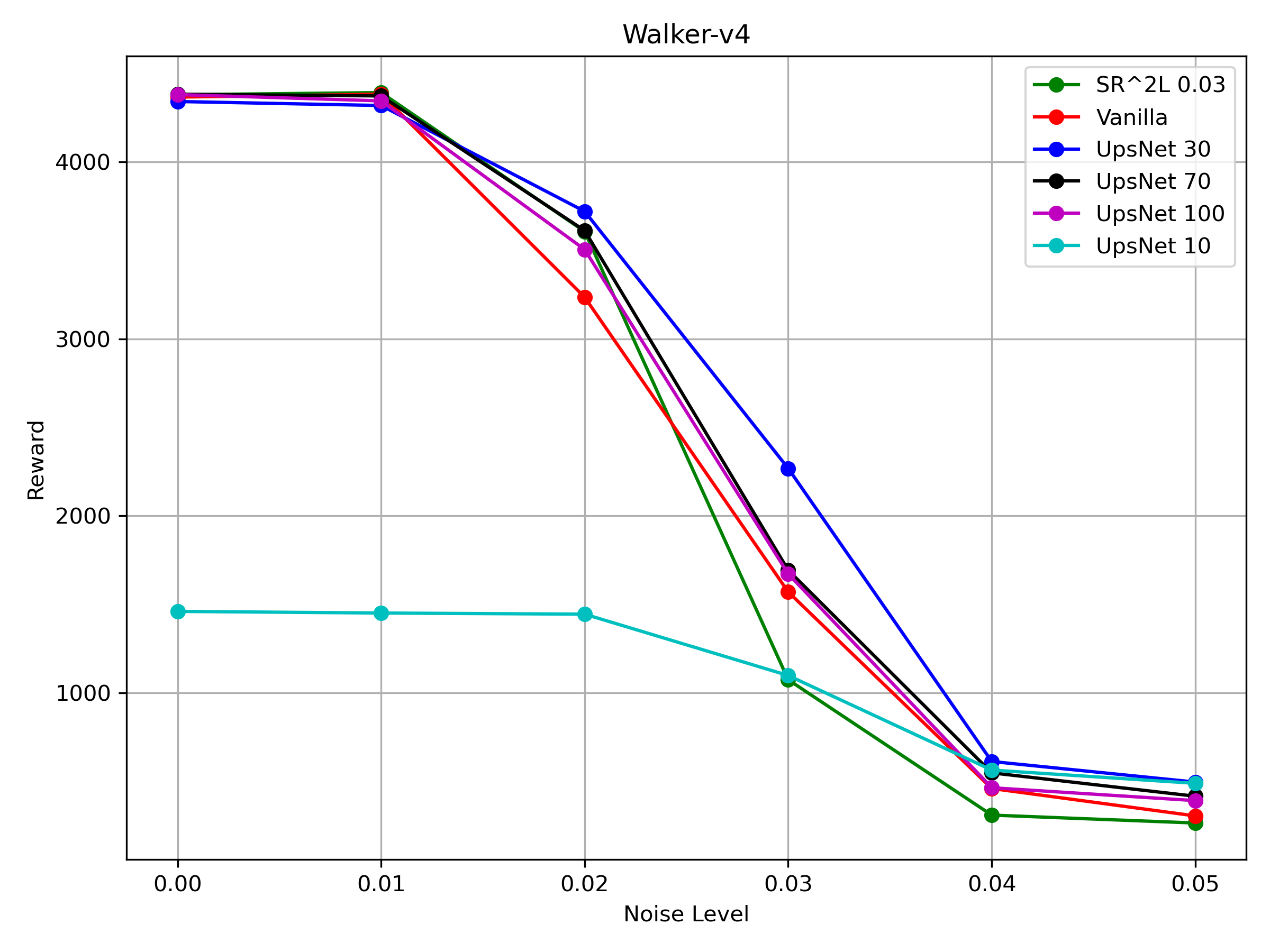}
    \caption*{(b) Walker2d-v4}
\end{minipage}
\hfill
\begin{minipage}[b]{0.24\textwidth}
    \centering
    \includegraphics[width=\textwidth]{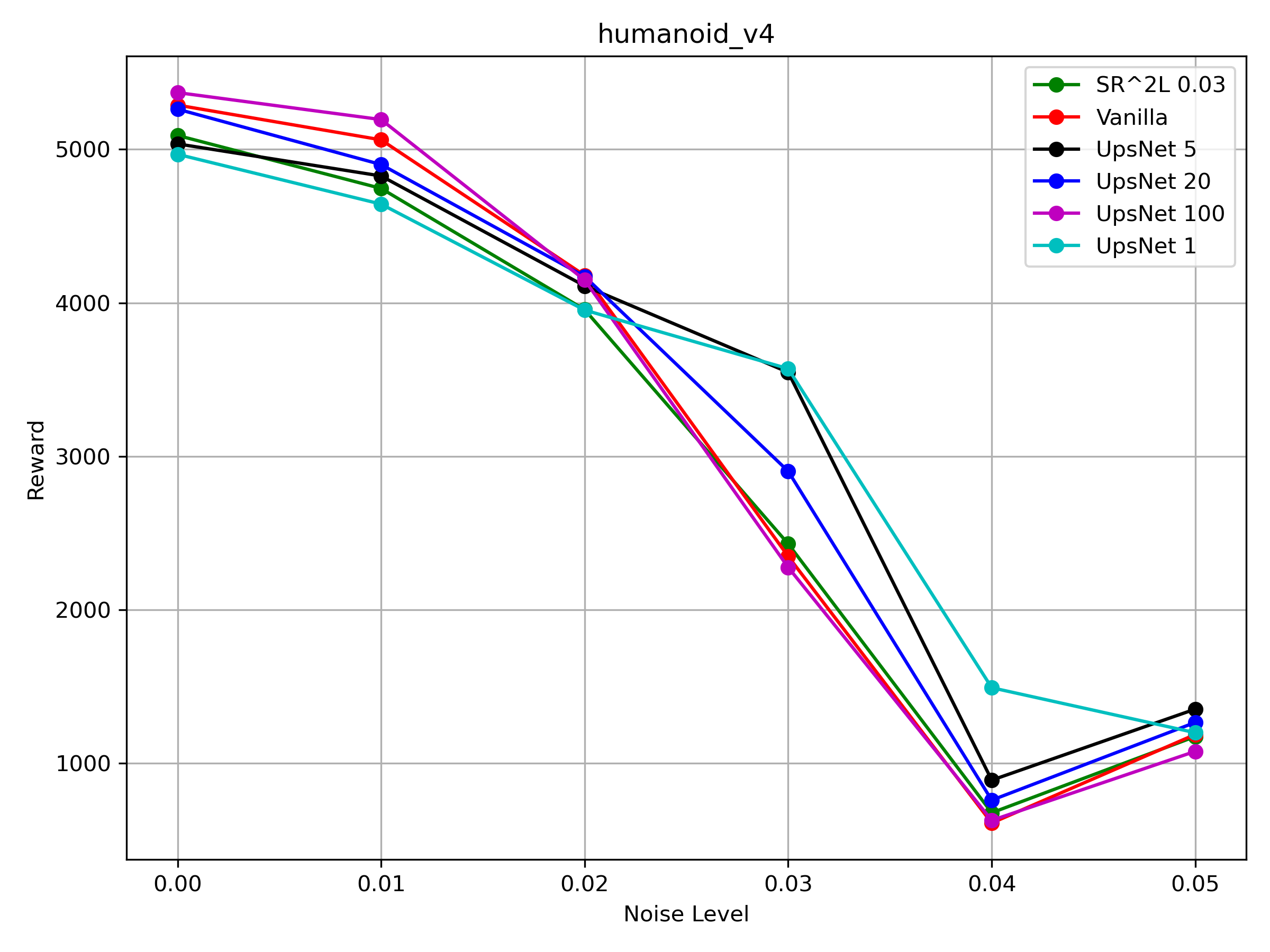}
    \caption*{(c) Humanoid-v4}
\end{minipage}
\hfill
\begin{minipage}[b]{0.24\textwidth}
    \centering
    \includegraphics[width=\textwidth]{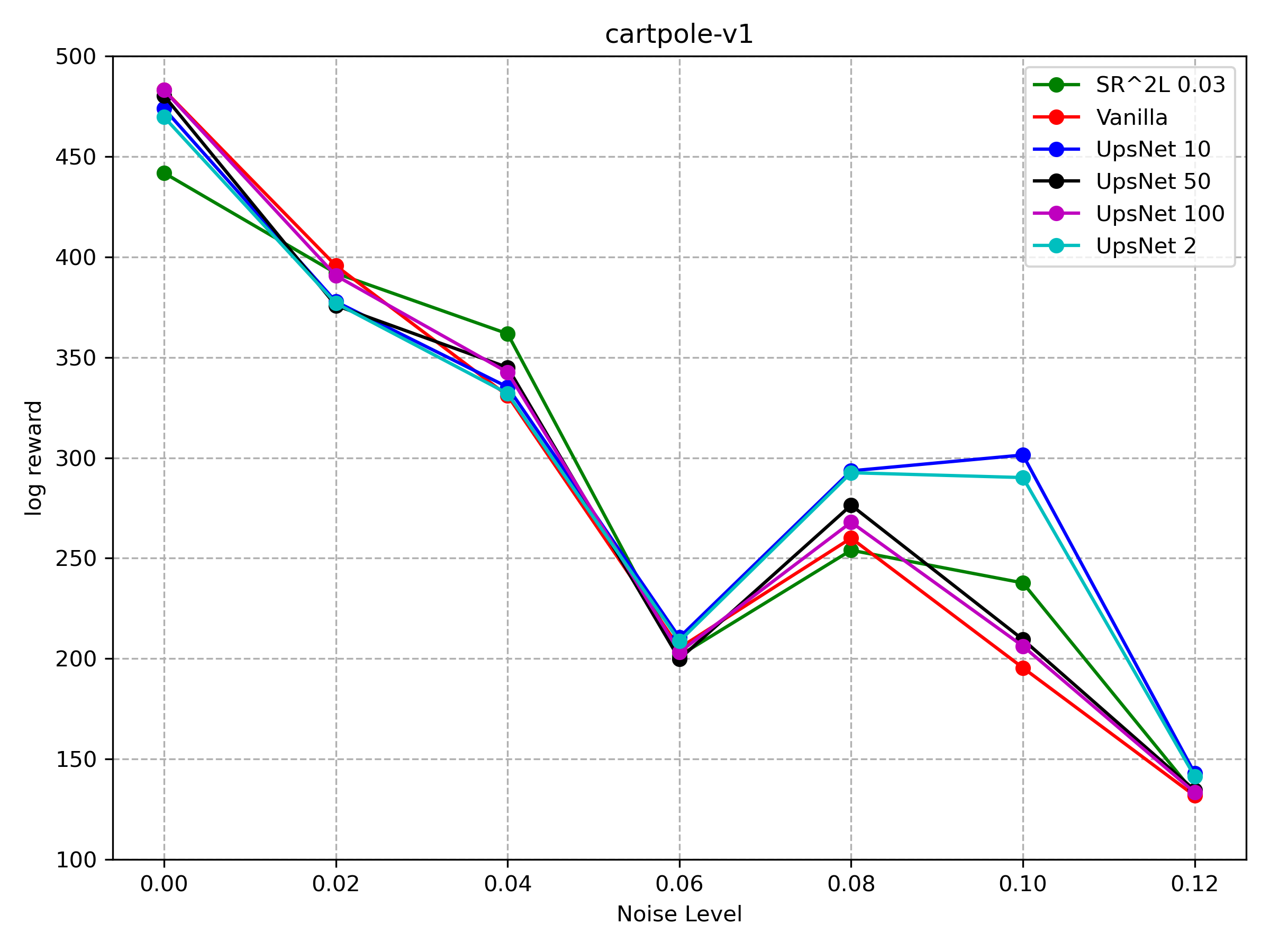}
    \caption*{(d) Cartpole-v0}
\end{minipage}

\vspace{1ex}
\caption{Worst-case mean reward policy-dependent noise at each noise level. (a–c) are deterministic policies; (d) is a stochastic policy.}
\label{fig:worst-average-reward}
\end{figure*}

\begin{figure*}[t]
\centering

\begin{minipage}[b]{0.24\textwidth}
    \centering
    \includegraphics[width=\textwidth]{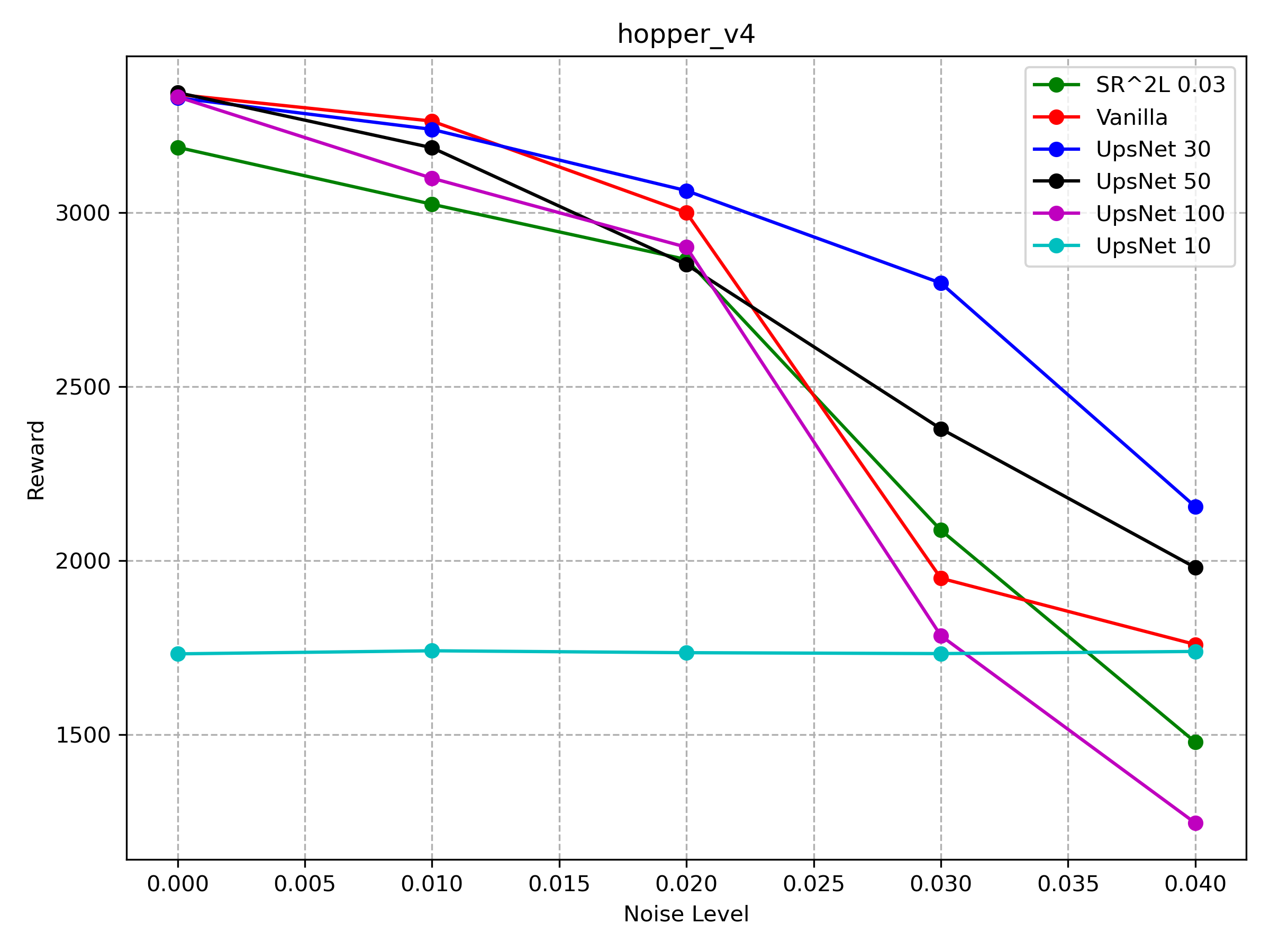}
    \caption*{(a) Hopper-v4}
\end{minipage}
\hfill
\begin{minipage}[b]{0.24\textwidth}
    \centering
    \includegraphics[width=\textwidth]{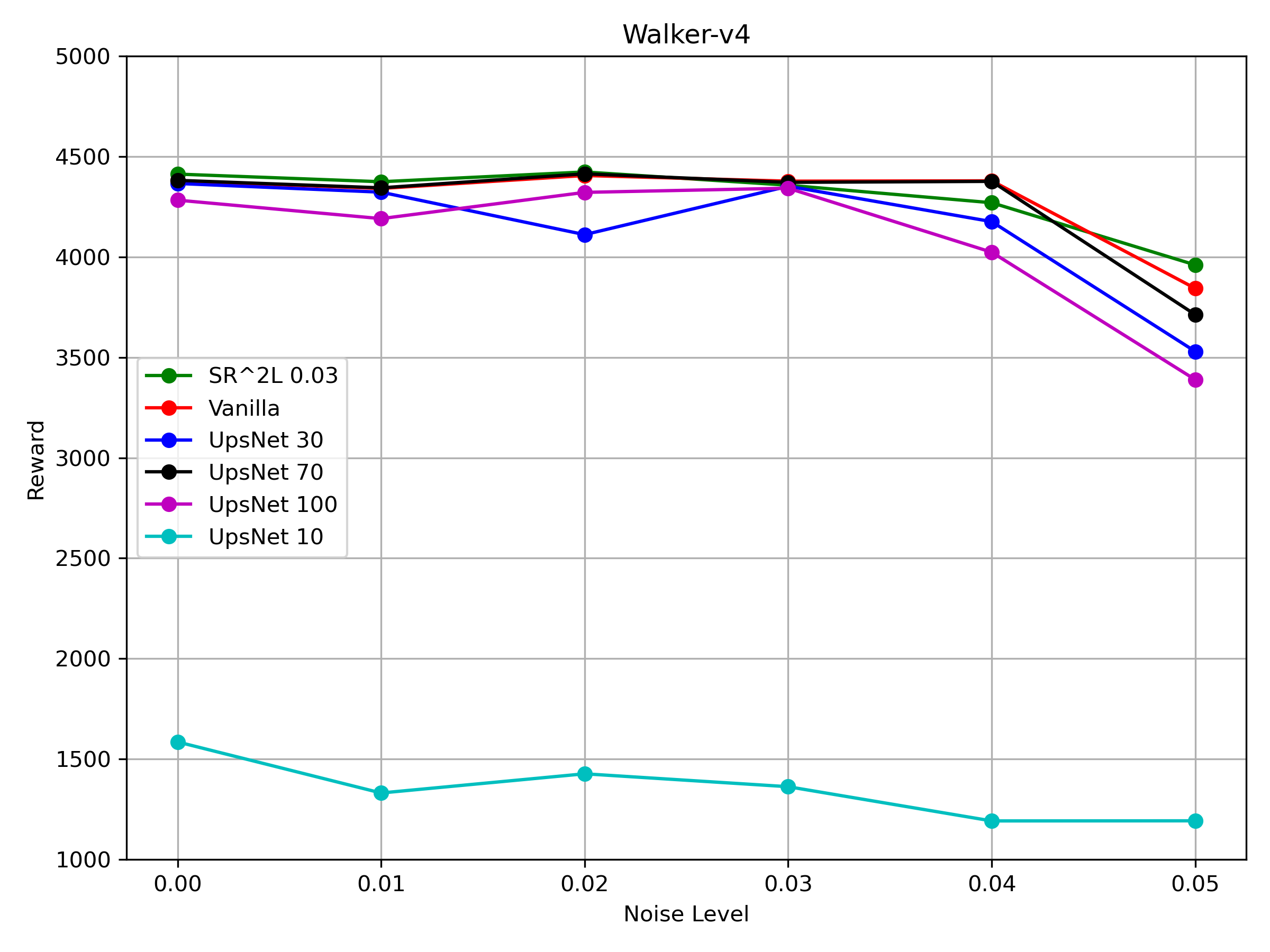}
    \caption*{(b) Walker2d-v4}
\end{minipage}
\hfill
\begin{minipage}[b]{0.24\textwidth}
    \centering
    \includegraphics[width=\textwidth]{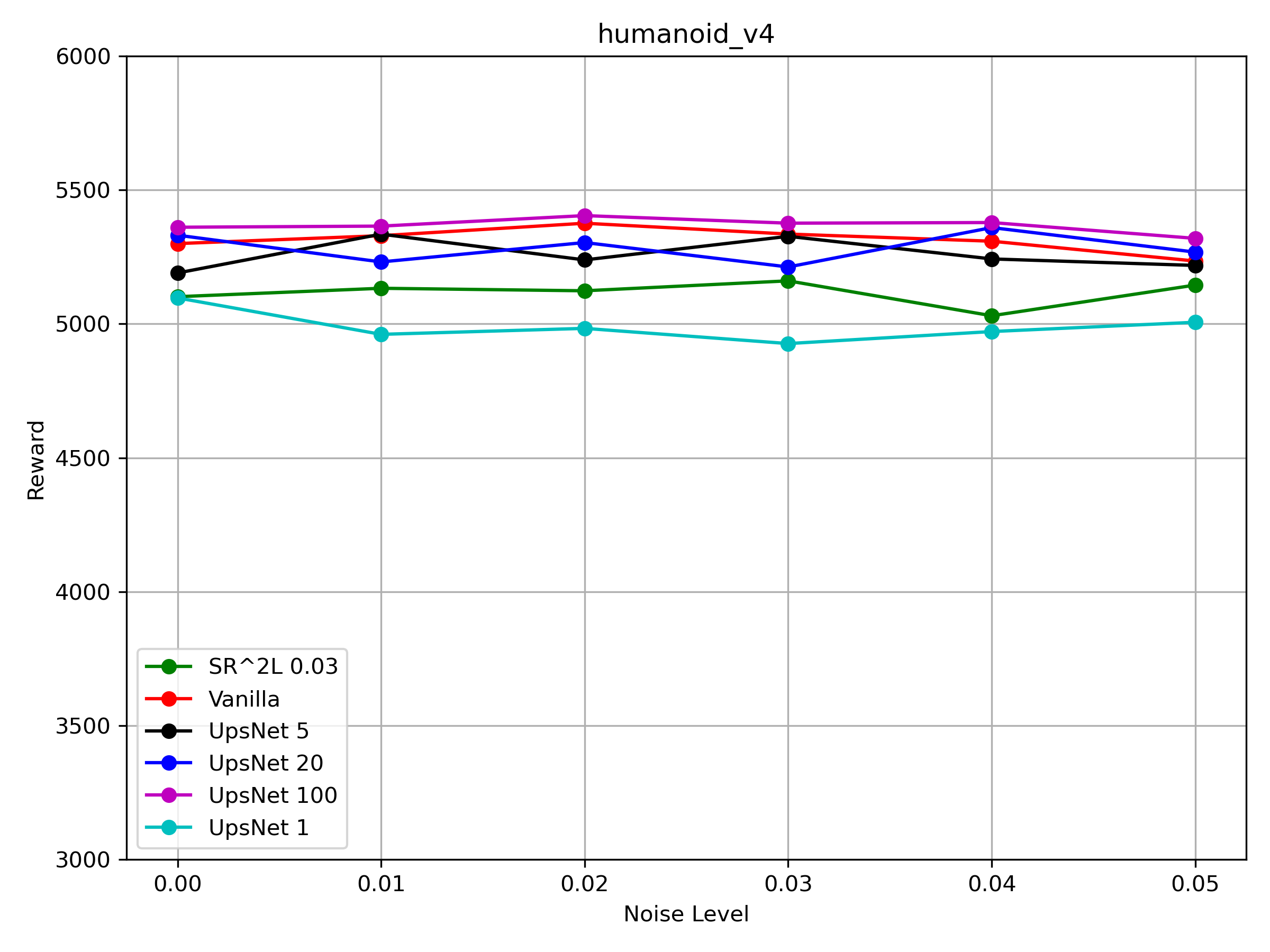}
    \caption*{(c) Humanoid-v4}
\end{minipage}
\hfill
\begin{minipage}[b]{0.24\textwidth}
    \centering
    \includegraphics[width=\textwidth]{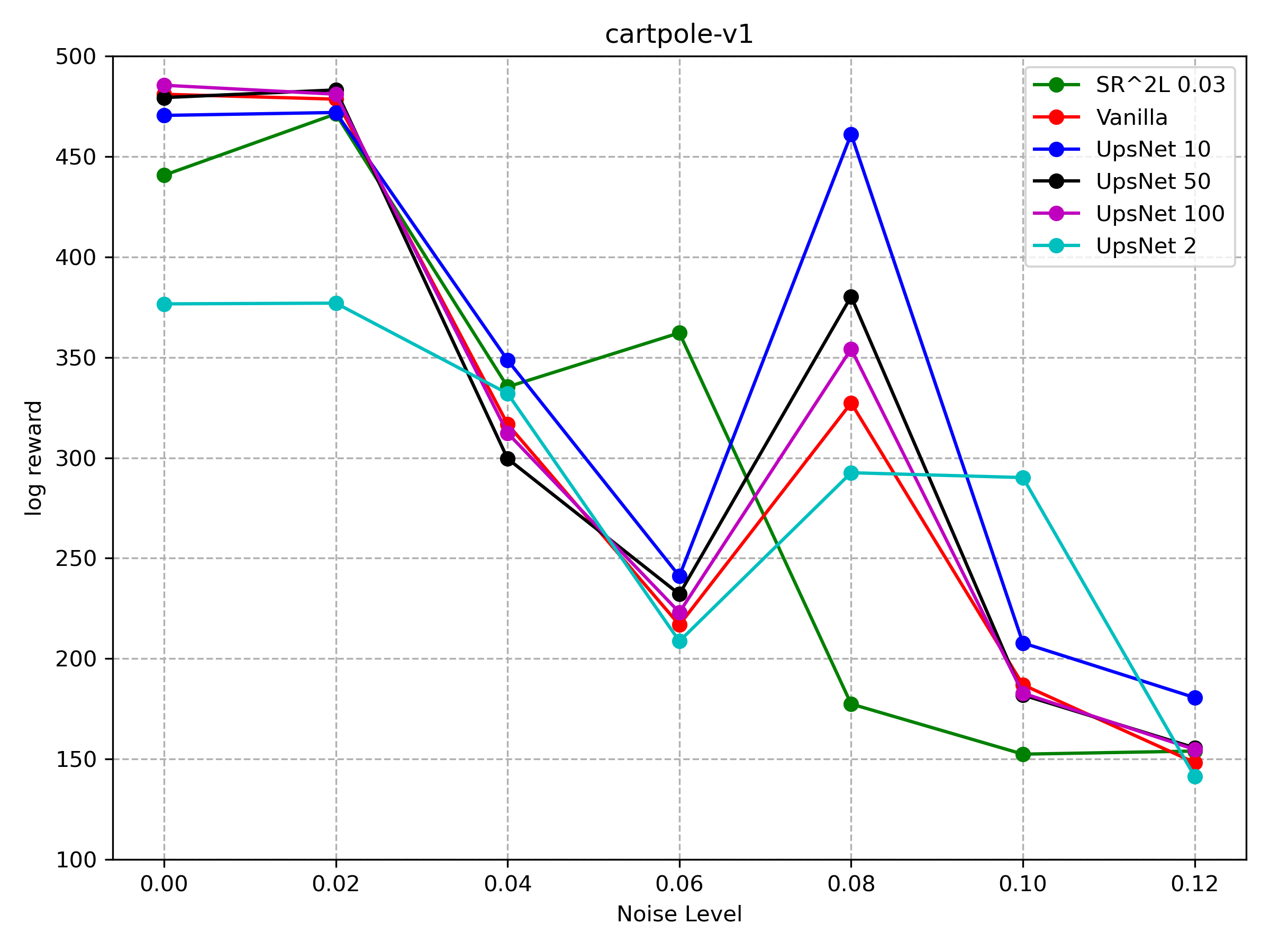}
    \caption*{(d) Cartpole-v0}
\end{minipage}

\vspace{1ex}
\caption{Worst-case mean reward policy-independent noise at each noise level. (a–c) are deterministic policies; (d) is a stochastic policy.}
\label{fig:worst-average-reward-independent}
\end{figure*}

\begin{table*}[t]
\begin{center}
\caption{Average reward of different algorithms facing $l_{\infty}$ and $l_2$ bounded noise. The number after LipsNet indicates the specified Lipschitz constant. SR$^2$L is trained with the corresponding $\epsilon$. We report mean $\pm$ std over 300 episodes. Best performing algorithm under each condition is highlighted.}
\label{tab:combined-performance-no-noise-uncolored}
\footnotesize
\begin{tabular}{|c|c|c|c|c|c|c|c|}
\hline
Environment & $\epsilon$ & Algorithm & No noise & Random & Adversarial & RS & Worst-case \\
\hline
\multirow{3}{*}{Hopper} & \multirow{3}{*}{0.03} & Vanilla & $3340\pm81$ & $2087\pm841$ & $889\pm496$ & $1949\pm924$ & $889 \pm 496$ \\
                        &                      & \cellcolor{purple!60}\textcolor{white}{LipsNet-50} & $3330\pm122$ & \cellcolor{purple!60}\textcolor{white}{$2647\pm854$} & \cellcolor{purple!60}\textcolor{white}{$1094\pm510$} & \cellcolor{purple!60}\textcolor{white}{$2379\pm682$} & \cellcolor{purple!60}\textcolor{white}{$1094 \pm 510$} \\
                        &                      & SR$^2$L & $3182\pm521$ & $1507\pm823$ & $873\pm462$ & $2087\pm1009$ & $873 \pm 462$ \\
\hline
\multirow{3}{*}{Walker2D} & \multirow{3}{*}{0.03} & Vanilla & $4378\pm11$ & \cellcolor{purple!60}\textcolor{white}{$4370\pm213$} & $1394\pm1551$ & $4377\pm43$ & $1394 \pm 1551$ \\
                          &                      & \cellcolor{purple!60}\textcolor{white}{LipsNet-30} & $4370\pm148$ & $4368\pm263$ & \cellcolor{purple!60}\textcolor{white}{$2268\pm1689$} & $4388\pm302$ & \cellcolor{purple!60}\textcolor{white}{$2268 \pm 1689$} \\
                          &                      & SR$^2$L & $4418\pm1112$ & $4347\pm546$ & $1379\pm1606$ & \cellcolor{purple!60}\textcolor{white}{$4405\pm22$} & $1379 \pm 1606$ \\
\hline
\multirow{3}{*}{Humanoid} & \multirow{3}{*}{0.03} & Vanilla & $5399\pm637$ & $5244\pm849$ & $2348\pm1317$ & $5334\pm995$ & $2348 \pm 1317$ \\
                          &                      & \cellcolor{purple!60}\textcolor{white}{LipsNet-5} & $5313\pm1409$ & $5313\pm1295$ & \cellcolor{purple!60}\textcolor{white}{$3548\pm1023$} & $5302\pm1065$ & \cellcolor{purple!60}\textcolor{white}{$3548 \pm 1023$} \\
                          &                      & SR$^2$L & $5449\pm498$ & \cellcolor{purple!60}\textcolor{white}{$5380\pm695$} & $2430\pm955$ & \cellcolor{purple!60}\textcolor{white}{$5398\pm610$} & $2430 \pm 955$ \\
\hline
\multirow{3}{*}{CartPole} & \multirow{3}{*}{0.10} & Vanilla & $483\pm35$ & \cellcolor{purple!60}\textcolor{white}{$480\pm36$} & $195\pm55$ & $186\pm7$ & $186 \pm 7$ \\
                          &                      & \cellcolor{purple!60}\textcolor{white}{LipsNet-10} & $473\pm43$ & $474\pm40$ & \cellcolor{purple!60}\textcolor{white}{$301\pm61$} & \cellcolor{purple!60}\textcolor{white}{$207\pm14$} & \cellcolor{purple!60}\textcolor{white}{$207 \pm 14$} \\
                          &                      & SR$^2$L & $440\pm65$ & $400\pm91$ & $235\pm133$ & $177\pm25$ & $177 \pm 25$ \\
\hline
\end{tabular}
\end{center}
\end{table*}

\section{Experiments} \label{sec:experiments}
\paragraph{Dataset Creation} To thoroughly evaluate the performance of our method, we conducted experiments in continuous and discrete action environments using the \textit{Walker2d-v4}, \textit{Hopper-v4} and \textit{Humanoid-v4} on \textit{MuJoCo} \citep{todorov2012mujoco} and \textit{Cartpole-v1} in Gymnasium \citep{towers_gymnasium_2023}. Due to the lack of a standard dataset for BC, we created our own by first training expert policies. We used the CleanRL \citep{huang2022cleanrl} implementation of Proximal Policy Optimization (PPO) \citep{DBLP:journals/corr/SchulmanWDRK17} to train the stochastic policy for Box2D and Twin Delayed DDPG (TD3) \citep{DBLP:journals/corr/abs-1802-09477} to train the deterministic policy for MuJoCo, using the default hyperparameters specified in \citep{huang2022cleanrl}. Each expert was trained for $10^6$ steps to ensure full convergence, achieving performance values matching those reported by the CleanRL project.
%as shown in "no-noise" under table. \ref{tab:combined-performance-no-noise-uncolored}.
%detailed in Appendix B.
%\ref{sec: expert performance}. 
Next, we built our own BC datasets by collecting 100 trajectories from the experts in each environment. For TD3, we removed action noise during training to make the policy purely deterministic during this sampling process. 

\paragraph{Training} Following standard practice, we divide each datasets into 70\% for training and 30\% for validation. We then trained the stochastic policies on Cartpole with Cross-Entropy (CE) loss and the deterministic policies on MuJoCo with Mean Square Error (MSE) loss in train datasets. In addition to the proposed Lipschitz regularization method and the vanilla method, we also considered the adversarial smooth regularization algorithm SR$^2$L \citep{shen2020deep} for comparison. We adapted SR$^2$L to adversarially train the local smooth regularizer. Detailed descriptions, training procedures, and additional hyperparameters are provided in Appendix B.
%\ref{sec: intro to noises}. 
Our algorithm, LipsNet, includes two hyperparameters: $L_\pi$ and $\lambda$. The parameter $L_\pi$ is designed to control robustness, potentially trading off optimality as discussed in the previous section. %Detailed effects of $L_\pi$ can be found in Appendix B. %\ref{sec: detailed reward curves}. 
Other hyperparameters are set to match those of vanilla behavior cloning. Unlike local smoothing regularization methods where $\lambda$ controls the robustness level implicitly \citep{zhang2020robust, zhao2022adversarially}, our $\lambda$ only controls the rate and the power of the regularization, which functions similarly to a learning rate. We set it to $\lambda = 0.001$ for all cases. Each BC policy was trained for $5 \times 10^5$ timesteps to ensure full convergence and we compute the evaluation score every $1000$ steps. Since BC is reward-agnostic, we selected the model with the best evaluation score for testing. Each method was trained with different random seeds four times to ensure an accurate evaluation.  
%The training curves are reported in Appendix B.
%\ref{sec: learning curve for deterministic action}

\paragraph{Testing} To test the efficacy of our algorithm against state perturbations, we used three types of bounded noise at different noise levels, $\delta$. Here, $\delta$ represents the magnitude of the noise applied to the state observation (with respect to the particular norm for the given setting). During the evaluation, agents observed a perturbed state at each step and executed actions based on this perturbed state. We evaluated the performance of our model at $\delta = [0.02, 0.04, 0.06, 0.08, 0.1, 0.12]$ in Cartpole and at $\delta = [0.01, 0.02, 0.03, 0.04, 0.05, 0.06]$ for the other environments.

We consider two types of perturbation. The first is policy-independent noise, which includes random noise, which are perturbations sampled from a zero mean distribution and bounded by either the $\ell_2$ or $\ell_\infty$ norm at a specified noise level, and Robust Sarsa (RS) \citep{zhang2020robust}, an online critic-based method that selects the action within a bounded set that minimizes the expected future return. The second type is policy-dependent noise, specifically Adversarial Noise generated during the training of SR$^2$L \citep{shen2020deep}. This adversarial model learns to produce bounded perturbations that maximize the difference in the policy's output actions, thereby inducing worst-case behavior.

% Details of these noise types can be found in Appendix B.
%\ref{sec: intro to noises}. 
We trained the adversarial and RS models four times at each noise level for the corresponding $2$-norm or $\infty$-norm. Our model was then evaluated 75 times on each trained adversarial and RS model. The mean and standard deviation of the results were reported.

\paragraph{Results}
Figure\ref{fig:worst-average-reward}, Figure\ref{fig:worst-average-reward-independent} and Table \ref{tab:combined-performance-no-noise-uncolored} provide a comparative evaluation of the worst-case performance across various policy learning methods under different types and levels of perturbations. As shown in Table~\ref{tab:combined-performance-no-noise-uncolored}, policy-dependent noise poses stronger attacks, as it is specifically trained to induce worst-case perturbations. Vanilla policies demonstrate poor robustness. In contrast, incorporating Lipschitz regularization during training results in policies that consistently achieve higher worst-case returns across a broad range of noise scenarios. It is worth noting that although the adversarial noise is generated through training an adversarial model, the resulting policies do not outperform those trained with Lipschitz regularization.

Figure \ref{fig:worst-average-reward} further illustrates that lower values of the Lipschitz constant are generally associated with greater robustness, as indicated by a smaller drop in performance relative to the noise-free case. However, this benefit comes with important caveats. If the regularization strength, which is represented by $L_\pi$, is set too high, the resulting constraint becomes too loose, causing the policy to behave similarly to the unregularized vanilla model. On the other hand, when $L_\pi$ is too small, the model can become overly conservative: Although it can maintain performance under perturbations, it often does so at the cost of its performance in noise-free environments. This trade-off between robustness and performance is similar to those in the adversarial training literature \cite{zhang2020robust}, where models trained for robustness against worst-case perturbations often exhibit reduced accuracy under noise-free settings. In general, our results show that global Lipschitz control is a principled and effective strategy to improve policy robustness, particularly valuable in safety-critical or high-stakes domains where noise and disturbances are inevitable.

\section{Discussion}
In this work, we initially investigated the relationship between the global Lipschitz constant and the robustness certificate of the policy. Following this, we introduced a robust behavior cloning (BC) method that utilizes weight normalization to control the Lipschitz upper bound of the neural network. This strategy aims to enhance the robustness of BC policies against observation perturbations by focusing on the network architecture. Our experimental results indicate that our method delivers more consistent robustness performance compared to some local smoothing techniques. Since our modification is applied at the network level and our discussion is based on a general case of the reinforcement learning scenario, it is anticipated that this approach can be extended to online reinforcement learning. However, due to potential inaccuracies in estimating the Lipschitz constant, there may be an issue of over-regularization, where the imitating policy encounters significant interpolation problems. Therefore, it is important to explore methods for regulating the neural network without introducing additional interpolation errors. Furthermore, while our findings show that robustness can be achieved by regulating the neural network under the $l_\infty$-norm it remains valuable to investigate the effects of regulating different weight norms.
\bibliography{aaai2026}

\clearpage
\onecolumn
%\documentclass[letterpaper]{article}
%\usepackage{times}
%\usepackage{helvet}
%\usepackage{courier}
%\usepackage{graphicx}
%\usepackage{amsmath}
%\usepackage{booktabs}
%\usepackage{url}
%\usepackage{hyperref}

% Single-column formatting
%\usepackage[margin=1in]{geometry}

%\usepackage{my_packages}

% Remove page numbers
\pagestyle{empty}

%\title{Supplemental Material for Anonymous AAAI 2026 Submission}

\author{Anonymous Author(s)}

%\begin{document}

\maketitle
\appendix
\section{Technical Proofs}
\subsection{Preliminary}
\label{sec: appendix prelim}
This section provides additional definitions and lemmas that will be used later.
\begin{definition}
    Consider an infinite horizon discounted MDP $\bM = (\bS, \bA, R, \bP, \gamma)$. Then the state-action function (or $Q$-function) under a policy $\pi$ is the expected discounted sum of rewards obtained under that policy starting from a state-action pair $(s, a)$, where $(s, a) \in \bS \times \bA$,
    \begin{align}
        Q^\pi(s, a) &= \E_\pi \left[\sum_{t=0}^{+\infty} \gamma^t r(s_t, a_t) | s_0 = s, a_0 = a\right].
    \end{align}
\end{definition}
\begin{definition} \cite{zhang2020robust}
    Consider an infinite-horizon discounted SA-MDP $\bM = (\bS, \bA, R, \bP, \gamma, B_\epsilon)$. Then the adversarial state-action value functions under a fixed $\mu$ are similar to that of a regular MDP and are given by
    \begin{align}
        Q^{\pi \circ \mu}(s, a) &= \E_{\pi \circ \mu} \left[\sum_{t=0}^{+\infty} \gamma^t r(s_t, a_t) | s_0 = s, a_0 = a\right]
    \end{align}
\end{definition}

\begin{lemma} \cite{sutton1999reinforcement}
    Given an infinite horizon MDP $\bM = (\bS, \bA, R, \bP, \gamma)$, the value function can be expressed in a recursive form
    \begin{align*}
        V^\pi(s) = \sum_a \pi(a|s)[r(s,a) +  \gamma \sum_{s'} \bP(s' | s, a) V^\pi(s')]
    \end{align*}
\end{lemma}
\begin{lemma} (Theorem 1 and a part of that proof in \cite{zhang2020robust}) \label{lemma: value function under noise}
    Given an infinite horizon MDP $\bM = (\bS, \bA, R, \bP, \gamma, B_\epsilon)$, a fixed policy $\pi: \bS \rightarrow \bP(\bA)$, and an adversary $\mu: \bS \rightarrow \bS$, we have
    \begin{align}
        V^{\pi \circ \mu}(s) &= \sum_a \pi(a|\mu(s)) \sum_{s'} p(s'|s,a) [R(s, a, s') + \gamma V^{\pi \circ \mu}(s')]\\
        V^{\pi \circ \mu}(s) &= \sum_a \pi(a|\mu(s)) Q^{\pi \circ \mu}(s, a)
    \end{align}
\end{lemma}

\subsection{Robust Certificate}
\label{proof: Lpi robust certificate}
\subsubsection{Robust Certificate for a Stochastic Policy}
We present our proof on drawing the robust certificate in the stochastic setting. Our proof is similar to that in \cite{zhang2020robust}. We start by giving the formal definition of the distance metric $\mathrm{D}_{TV}$, then show a universal bound on the value function in prop. \ref{prop:max V, Q}, and then finally use it with the performance difference lemma and the definition of Lipschitzness to show the thm. \ref{thm: robust cerficate for stochastic policy}.
\begin{definition}(prop 4.2 in \cite{levin2017markov})
    \label{def: TV}
    Let $\Omega$ be a set and $\Sigma$ be a $\sigma$-algebra over $\Omega$. We denote with $\bP(\Omega)$ the set of probability measures over the measurable space $(\Omega, \Sigma)$. Then for any two probability measures $\mu, \nu \in \bP(\Omega)$, the $TV(\cdot, \cdot)$ is defined as 
    \begin{align}
        \mathrm{D}_{TV}(\mu, \nu) = \frac{1}{2}\sum_{x \in \Omega} |\mu(x) - \nu(x)| = \frac{1}{2} \Norm{\mu(x) - \nu(x)}_1
    \end{align}
\end{definition} 

\begin{proposition} \label{prop:max V, Q}
    Consider an infinite-horizon discounted MDP $\bM = \{\bS, \bA, R, \bP, \gamma \}$, with a bounded reward function $\Norm{R}_\infty \in [0, R_{\max}]$ and an arbitrary policy $\pi$. Then, $\forall (s, a) \in \bS \times \bA$, $V^\pi(s) \leq \frac{1}{1-\gamma} R_{\max}$ and $Q^\pi(s, a) \leq \frac{1}{1-\gamma} R_{\max}$ 
\end{proposition}
\begin{proof}
    By definition of $V^\pi$, we have
    \begin{align*}
        V^\pi(s) = \E_{\pi} \left[\sum_{t=0}^{+\infty} \gamma^t R_t | s_0 = s\right] \leq \E_{\pi} \left[\sum_{t=0}^{+\infty} \gamma^t R_{\max} | s_0 = s\right] =  R_{\max} \sum_{t=0}^{+\infty} \gamma^t = \frac{1}{1-\gamma} R_{\max}
    \end{align*}
    Following the same procedure, we can derive the bound for $Q^\pi(s,a)$ as well. 
\end{proof}
\begin{theorem}
\label{thm: robust cerficate for stochastic policy}
    Consider an infinite horizon SA-MDP $\bM = \{ \bS, \bA, R, \bP, \gamma, B_\epsilon \}$ with a bounded reward function $0 \leq R \leq R_{\max}$ and an $\pi$ with a local lipschitz function $L^\epsilon_\pi (s)$. Then, for small enough $\epsilon$ and any fixed bounded adversary $\mu: \bS \rightarrow \bS$
    \begin{align*}
        \Theta(\pi) = \max_s [V^\pi(s) - V^{\pi \circ \mu}(s)] \leq \frac{R_{\max}}{(1-\gamma)^2} L^\epsilon_\pi(s) \epsilon
    \end{align*}
\end{theorem}
\begin{proof} 
    Inspired by the proof presented in \cite{maran2022tight}, we start from the Performance Difference Lemma \cite{Kakade2002ApproximatelyOA}, where the advantage function $A^\pi(s, a)$ of a policy $\pi$ is defined by $A^\pi(s, a) = Q^\pi(s, a) - V^\pi(s)$:
    \begin{align*}
        V^\pi(s) - V^{\pi \circ \mu}(s) &= \frac{1}{1-\gamma} \mathbb{E}_{s \sim d^\pi} [\mathbb{E}_{a \sim \pi(\cdot|s)} [A^{\pi \circ \mu}(s, a)]]\\
        &= \frac{1}{1-\gamma} \mathbb{E}_{s \sim d^\pi} [\mathbb{E}_{a \sim \pi(\cdot|s)} [Q^{\pi \circ \mu}(s, a) - V^{\pi \circ \mu}(s)]]\\
        &\overset{(i)}{=} \frac{1}{1-\gamma} \mathbb{E}_{s \sim d^\pi} [\mathbb{E}_{a \sim \pi(\cdot|s)} [Q^{\pi \circ \mu}(s, a)] - \E_{a \sim \pi(\cdot|\mu(s)} [Q^{\pi \circ \mu}(s, a)]]\\
        &\overset{(ii)}{=} \frac{1}{1-\gamma} \mathbb{E}_{s \sim d^\pi} \sum_{a \in \bA} Q^{\pi \circ \mu}(s, a) (\pi(a|s) - \pi(a|\mu(s)) \\
        &\leq \frac{1}{1-\gamma} [\max_{(s, a) \in \bS \times \bA} Q^{\pi \circ \mu}(s, a)] \mathbb{E}_{s \sim d^\pi} \sum_{a\in \bA} |\pi(a|s) - \pi(a|\mu(s))|\\
        &\overset{(iii)}{=} \frac{1}{1-\gamma} [\max_{(s, a) \in \bS \times \bA} Q^{\pi \circ \mu} (s, a)] \mathbb{E}_{s \sim d^\pi} \Norm{\pi(a|s) - \pi(a|\mu(s))}_{TV}\\
        &\overset{(iv)}{\leq} \frac{R_{\max}}{(1-\gamma)^2} \mathbb{E}_{s \sim d^\pi} L_\pi d_\bS(s, \mu(s)) \overset{(v)}{\leq} \frac{R_{\max}}{(1-\gamma)^2} L^\epsilon_\pi(s) \epsilon
    \end{align*}
    where $(i)$ is by lemma \ref{lemma: value function under noise}, $(ii)$ is by expanding the expectation, $(iii)$ is by the definition of TV distance, $(iv)$ is by proposition \ref{prop:max V, Q} and the definition of local smoothness function $L^\epsilon_\pi(s)$ for a policy $\pi$, and the last inequality follows from the fact that the weighted average is upper bounded by the maximum value. 
\end{proof}

\subsubsection{Robust Certificate for a Deterministic Policy}
We consider deterministic policies to be stochastic policies $\pi: (\mathbb{R}^n, \Norm{\cdot}_p) \rightarrow (\delta(\mathbb{R}^m), \bW(\cdot, \cdot))$, where $\delta(\cdot)$ is a Dirac measure in the corresponding action space $\mathbb{R}^m$ and the $\bW(\cdot, \cdot)$ is Wasserstein Distance (Definition \ref{def: W distance}). We then use Theorem \ref{thm: robust certificate under W} to show that an $L_\pi^W$-LC (Definition \ref{def: LpiW}) policy can provide a robustness certificate, which has a similar form to that of Theorem \ref{thm: performance drop}. %Finally, we can use this knowledge to construct the corresponding neural network using Theorem \ref{thm: neural network design in W}.

\begin{definition} \label{def: W distance}
    Let $\Omega$ be a set and $\Sigma$ be a $\sigma$-algebra over $\Omega$. We denote with $\bP(\Omega)$ the set of probability measures over the measurable space $(\Omega, \Sigma)$. For any two probability measures $\mu, \nu \in \bP(\Omega)$ and a cost function $f(\omega)$, the $L_1$-Wasserstein distance \cite{villani2009optimal} can be defined as:
    \begin{align}
        \bW(\mu, \nu) = \sup_{\Norm{f}_L \leq 1} \left|\int_\Omega f(\omega)(\mu - \nu)d\omega\right|
    \end{align}
\end{definition}
\begin{definition} \label{def: LpiW}
    A policy $\pi$ is said to be $L_\pi^\bW$-continuous if $\forall s, s' \in \bS$, $\bW(\pi(s), \pi(s'))\leq L_\pi^\bW D_\bS(s, s')$
\end{definition}
\begin{lemma} \label{thm: neural network design in W}
    If a neural network $f: (\mathbb{R}^n, \Norm{\cdot}_p) \rightarrow (\mathbb{R}^m, \Norm{\cdot}_p)$ is $L_\pi$-LC, then it is also $L_\pi$-LC in $(\mathbb{R}^n, \Norm{\cdot}_p) \rightarrow (\delta(\mathbb{R}^m), \bW(\cdot, \cdot))$
\end{lemma}
\begin{proof}
    $\forall x, y \in \mathbb{R}^n$
    \begin{align*}
        \bW(\delta(f(x)), \delta(f(y))) &= \Norm{f(x) - f(y)}_p \leq L_\pi\Norm{x-y}_p
    \end{align*}
    The first equality is because if two distributions are Dirac measures at $x, y$, then $\bW(\delta(x), \delta(y)) = d_X(x, y)$ \cite{maran2022tight}.
\end{proof}

\begin{lemma} (Generalized version of Theorem 3.1 in \cite{bukharin2023robust})
\label{lemma: Lips of Q function}
    Consider an infinite horizon MDP $\bM = \{ \bS, \bA, R, \bP, \gamma \}$ with $L_r$-LC reward function $r$ such that $|r(s, a) - r(s', a')| \leq L_r[\mathrm{D}_\bA(a, a') + \mathrm{D}_\bS(s, s')]$. Then the Q-function of any policy $\pi$ is $L_Q$-LC, that is
    \begin{align}
        |Q(s,a) - Q(s', a')| \leq L_Q [\mathrm{D}_\bA(a, a') + \mathrm{D}_\bS(s, s')]
    \end{align}
    where $L_Q := L_r + \frac{\gamma R_{max}}{1 - \gamma}$.
\end{lemma}
\begin{proof}
    \begin{align*}
        |Q^\pi(s,a) - Q^\pi(s', a')| &= |r(s, a) - r(s', a') + \gamma \sum_{s_1 \in \bS} P(s_1|s, a) V^\pi(s_1) - \gamma \sum_{s_1 \in \bS} P(s_1|s', a') V^\pi(s_1)|\\
        &\leq |r(s, a) - r(s', a')| + \gamma |\sum_{s_1 \in \bS} P(s_1|s, a) V^\pi(s_1) - \sum_{s_1 \in \bS} P(s_1|s', a') V^\pi(s_1)| \\
        &\leq |r(s, a) - r(s', a')| + \gamma \Norm{V^\pi(s_1)}_\infty  \sum_{s_1 \in \bS}|P(s_1|s, a) - P(s_1|s', a')|\\
        &\leq L_r(\mathrm{D}_\bS (s, s') + \mathrm{D}_\bA (a, a')) + \frac{\gamma R_{max}}{1 - \gamma} (\mathrm{D}_\bS (s, s') + \mathrm{D}_\bA (a, a'))\\
        &\leq [L_r + \frac{\gamma R_{max}L_P}{1 - \gamma}] (\mathrm{D}_\bS (s, s') + \mathrm{D}_\bA (a, a'))
    \end{align*}
\end{proof}

\begin{lemma}
    Consider two probability measures $\mu(x)$, $\nu(x)$ and an $L_c$-Lipschitz continuous cost function $c(x)$. Then
    \begin{align}
        \int_{x \in \mathcal{X}} c(x)(\mu(x) - \nu(x)) dx \leq L_c \bW(\mu, \nu)
    \end{align}
\end{lemma}
\begin{proof}
    \cite{rachelson2010locality}
    \begin{align}
        \int_{x \in \mathcal{X}} c(x)(\mu(x) - \nu(x)) dx &\leq |\int_{x \in \mathcal{X}} c(x)(\mu(x) - \nu(x)) dx|\\
        &\leq |\int_{x \in \mathcal{X}} c(x)[\mu(x) - \nu(x)] dx| \\
        &\leq L_c |\int \frac{c(x)}{L_c}[\mu(x) - \nu(x)] dx|\\
        &\leq L_c \sup_{\Norm{f}_L \leq 1} |\int f(x) \cdot (\mu(x) - \nu(x)) dx| = L_c \bW(\mu, \nu)
    \end{align}
    The last inequality follows from the fact that $\Norm{f(x)}_L := \Norm{\frac{c(x)}{L_c}}_L \leq 1$ since $L_c := \sup_x \|c(x)\|_L$. The last equality follows from the definition of the Wasserstein distance.
\end{proof}
\begin{theorem} \label{thm: robust certificate under W}
    Consider an infinite horizon ($L_P, L_r$)-LC SA-MDP $\bM = \{ \bS, \bA, R, \bP, \gamma, B_\epsilon \}$ with a bounded reward function $0 \leq R \leq R_{\max}$
    and $\pi$ with a local Lipschitz function $L^\epsilon_\pi (s)$. Then, for small enough $\epsilon$, any fixed bounded adversary $\mu: \bS \rightarrow \bS$, and any $\hat{s} \in B_\epsilon$
    \begin{align*}
        \Theta(\pi) = \max_s [V^\pi(s) - V^{\pi \circ \mu}(s)] \leq \alpha L^\epsilon_\pi(s) \epsilon
    \end{align*}
    where $\alpha := \frac{1}{(1-\gamma)} (L_r + \frac{\gamma R_{max}}{1 - \gamma}L_P)$. 
\end{theorem}
\begin{proof}
   Inspired by the proof presented in \cite{maran2022tight}, we start from the Performance Difference Lemma: 
    \begin{align*}
        V^\pi(s) - V^{\pi \circ \mu}(s) &= \frac{1}{1-\gamma} \mathbb{E}_{s \sim d^\pi} [\mathbb{E}_{a \sim \pi(\cdot|s)} [A^{\pi \circ \mu}(s, a)]]\\
        &= \frac{1}{1-\gamma} \mathbb{E}_{s \sim d^\pi} [\mathbb{E}_{a \sim \pi(\cdot|s)} [Q^{\pi \circ \mu}(s, a) - V^{\pi \circ \mu}(s)]]\\
        &= \frac{1}{1-\gamma} \mathbb{E}_{s \sim d^\pi} [\mathbb{E}_{a \sim \pi(\cdot|s)} [Q^{\pi \circ \mu}(s, a)] - \E_{a \sim \pi(\cdot|\mu(s)} [Q^{\pi \circ \mu}(s, a)]]\\
        &= \frac{1}{1-\gamma} \mathbb{E}_{s \sim d^\pi} \sum_{a \in \bA} Q^{\pi \circ \mu}(s, a) (\pi(a|s) - \pi(a|\mu(s)) \\
        &\leq \frac{1}{1-\gamma} \Norm{Q^{\pi \circ \mu}(s, \cdot)}_L \mathbb{E}_{s \sim d^\pi} \bW(\pi, \pi \circ \mu)\\
        &\leq \frac{1}{1-\gamma}\Norm{Q^{\pi \circ \mu}(s, \cdot)}_L\mathbb{E}_{s \sim d^\pi} L_\pi d_\bS(s, \mu(s))\\
        &\leq \frac{1}{(1-\gamma)} (L_r + \frac{\gamma R_{max}}{1 - \gamma}L_P) L^\epsilon_\pi(s) \epsilon
    \end{align*}
The last inequality is due to Lemma \ref{lemma: Lips of Q function}
\end{proof}
%A notable distinct between this bond and the bound in \cite{maran2022tight} is that 
%\begin{proposition} \label{prop:min Q, V}
%    Under a infinite discounted MDP $\bM = \{ \bS, \bA, R, P, \gamma \}$, with a bounded reward function $R \in [0, R_{\max}]$ an arbitrary policy $\pi$ and an arbitrary transition probability $\bP$, $\forall (s, a) \in \bS \times \bA, V_\bP^\pi(s) \geq 0$ and $Q_\bP^\pi(s, a) \geq 0$ 
%\end{proposition}
%\begin{proof}
%    obviously
%\end{proof}
\subsection{Lipschitzness of the Neural Network Based Policy}
\label{sec: lipschitzness of the policy}
\subsubsection{Lipschitz of Neural Network based Policy on Deterministic and Continuous Action}
\label{sec: lips bound on deterministic}
We first show that regulating the $\Norm{\cdot}_\infty$ of a neural network is sufficient to provide guarantees for observation noises bounded by different norms in $\mathbb{R}^n$.

\begin{lemma} \label{Lemma: norm equivalence}(Norm Equivalence in finite dimension). If $x \in(\mathbb{R}^n, \Norm{\cdot}_p)$, then
    \begin{align*}
        \Norm{x}_\infty \leq \Norm{x}_2 \leq \Norm{x}_1 \leq n \Norm{x}_\infty
    \end{align*}
\end{lemma}
The norm equivalence enables us to determine the Lipschitz constant of a neural network while employing different distance metrics in its domain.
\begin{lemma}
    \label{thm: Lpi equavilence in neural network}
    Consider a fully connected neural network $f: (\mathbb{R}^n, \Norm{\cdot}_\infty) \rightarrow (\mathbb{R}^m, \Norm{\cdot}_\infty)$ that is $L_\pi^{FC}$-LC. Then $\forall x, y \in \mathbb{R}^n$
    \begin{align*}
        \Norm{f(x) - f(y)}_\infty \leq L_\pi^{FC} \Norm{x-y}_\infty \leq L_\pi^{FC} \Norm{x - y}_p .
    \end{align*}
\end{lemma} 

Also, by \cite{liu2022learning}, the Lipschitzness of an $M$ layer fully connected neural network can be inferred from the following lemma.
\begin{lemma}\label{lemma: liptchitz bound of fully connected}
    Consider an $M$ layer fully connected neural network $f: (\mathbb{R}^n, \Norm{\cdot}_\infty) \rightarrow (\mathbb{R}^m, \Norm{\cdot}_\infty)$, where the $i$-th layer has weights $W_i$, bias $b_i$ and a 1-Lipschitz activation function (e.g Relu) $\sigma$. Then $\forall x, y \in \mathbb{R}^n$
    \begin{align*}
        \Norm{f(x) - f(y)}_\infty \leq L^{FC}_\pi \Norm{x-y}_\infty
    \end{align*}
where $L^{FC}_\pi = \prod_0^M \Norm{W_i}_\infty$ and the $\infty$-norm for the matrix denotes the induced matrix norm. 
\end{lemma}
Then by lemma \ref{thm: Lpi equavilence in neural network} and  lemma \ref{lemma: liptchitz bound of fully connected}, it is easy to show the following.
\begin{theorem}
    Consider an $M$ layer fully connected neural network $f: (\mathbb{R}^n, \Norm{\cdot}_\infty) \rightarrow (\mathbb{R}^m, \Norm{\cdot}_\infty)$, where the $i$-th layer has weights $W_i$, bias $b_i$ and a 1-liptchitz activation function (e.g Relu) $\sigma$. Then $\forall x, y \in \mathbb{R}^n$
    \begin{align*}
        \Norm{f(x) - f(y)}_{2} \leq m\Norm{f(x) - f(y)}_\infty \leq m L^{FC}_\pi \Norm{x-y}_p
    \end{align*}
where $L^{FC}_\pi = \prod_0^M \Norm{W_i}_\infty$ and the $\infty$-norm for the matrix denotes the induced matrix norm. 
\end{theorem}

\subsubsection{Lipschitz of Neural Network based Policy on Discrete Action Space}
\label{sec: lips bound on softmax}
\begin{lemma}\label{proof: liptchitz bound on softmax}
    Consider a vector valued function $softmax(\cdot): \mathbb{R}^m \rightarrow \Delta(\mathbb{R}^m)$, such that $[softmax(x)]_i = [\frac{e^{x_i}}{\sum_j e^{x_j}}]$, where subscript $i$ indicates the $i$-th component of the corresponding vector. Then $\forall x, y \in \mathbb{R}^m$
    \begin{align*}
        \mathrm{D}_{TV}(softmax(x), softmax(y)) \leq \frac{m}{2}\Norm{x - y}_1
    \end{align*}
\end{lemma}
\begin{proof}
    The proof is partially drawn from \cite{bukharin2023robust}. Clearly, $\sum_i [softmax(x)]_i = 1$ and $[softmax(x)]_i \geq 0$, so $softmax(x)$ can be used to output a probability mass function. 
    
    We denote the Jacobian matrix of a $softmax$ function at point x as $\mathcal{J}_x$. Then
    \begin{align*}
        &[\mathcal{J}_x]_{i,i} = \frac{\exp(x_i) \sum_{j \neq i} \exp(x_j)}{(\sum_j \exp(x_j))^2} = \frac{\exp(x_i)}{\sum_j \exp(x_j))} \frac{\sum_{j \neq i} \exp(x_j)}{\sum_j \exp(x_j)}\\
        &[\mathcal{J}_x]_{i,j} = -\frac{\exp(x_i) \exp(x_j)}{(\sum_j \exp(x_j))^2} = - \frac{\exp(x_i)}{\sum_j \exp(x_j)} \frac{\exp(x_j)}{\sum_j \exp(x_j)}
    \end{align*}
    Clearly, we have $|[\mathcal{J}_x]_{i,i}| < 1$ and $|[\mathcal{J}_x]_{i,j}| \leq 1$ so $\Norm{J_z}_1 \leq m$. Then, $\forall x, y \in \mathcal{R}^m$, 
    we obtain
    \begin{align*}
        \mathrm{D}_{TV}(Softmax(x), Softmax(y)) &= \frac{1}{2}\sum_i |[softmax(x)_i] - [softmax(y)]_i|\\
        &= \frac{1}{2}\Norm{Softmax(x) - Softmax(y)}_1\\
        &\leq \frac{1}{2} \Norm{J_z}_1 \Norm{x-y}_1 \mbox{ for }z = \alpha x + (1 - \alpha )y \mbox{ and } \alpha \in  [0, 1]  \\
        &\leq \frac{m}{2}\Norm{x-y}_1.
    \end{align*}
    where the first inequaility is from the Intermediate Value Theorem
\end{proof}

\begin{theorem} \label{proof: Lpi of a neural network}
    Consider a neural network $f:(\mathbb{R}^n, \Norm{\cdot}_p) \rightarrow (\Delta(\mathbb{R}^m), TV)$ with two components: a fully connected network $f^{FC}: (\mathbb{R}^m, \Norm{\cdot}_\infty) \rightarrow (\mathbb{R}^n, \Norm{\cdot}_\infty)$ and a $softmax(\cdot)$ as the output layer. Then $\forall x, y \in \mathbb{R}^n$
    \begin{align*}
        \mathrm{D}_{TV}(f(x), f(y)) \leq \frac{m^2}{2} L^{FC}_\pi \Norm{x - y}_p.
    \end{align*}
\end{theorem}
\begin{proof}
    \begin{align*}
        \mathrm{D}_{TV}(f(x), f(y)) &= \mathrm{D}_{TV}(softmax(f^{FC}(x)), softmax(f^{FC}(y)))\\
        &\leq \frac{m}{2} \Norm{f^{FC}(x) - f^{FC}(y)}_1\\
        &\leq \frac{m^2}{2}\Norm{f^{FC}(x) - f^{FC}(y)}_\infty\\
        &\leq \frac{m^2}{2} L^{FC}_\pi \Norm{x - y}_\infty\\
        &\leq \frac{m^2}{2} L^{FC}_\pi \Norm{x - y}_p
    \end{align*}
  In the above proof, the first inequality is due to Lemma \ref{proof: liptchitz bound on softmax}, the second inequality is due to Lemma \ref{Lemma: norm equivalence}, the third inequality is due to the Definition of $L_\pi^{FC}$ and the fourth inequality is due to lemma \ref{thm: Lpi equavilence in neural network}.
\end{proof}
\subsubsection{Lipschitz of Neural Network for a General Stochastic Policy}
\label{sec: lip bound on general stochastic}
\begin{theorem}
    Consider a neural network based policy with a Lipschitz constant $L_\pi^{FC}$ that outputs the mean $\mu$ and the variance $\Sigma$ of a Gaussian policy, such that $\pi(a|s) \sim \mathcal{N}(\mu_s, \Sigma_s)$, by assuming the $\Sigma$ is independent to $s$, then the policy $\pi$ satisfies
    \begin{align}
        \mathrm{D}_{TV} (\pi(a|s), \pi(a|s')) \leq \mathrm{D}_{KL} (\pi(a|s) || \pi(a|\hat{s})) \leq L_\pi \Norm{s - s'}_p
    \end{align}
    where $L_\pi := kL_\pi^{FC}\Norm{\Sigma}_2$ and $k$ is a constant that depends on the choice of $p$-norm.
\end{theorem}
\begin{proof}
    we start by using the close form provided by \cite{zhang2020robust}, and following the same assumption on the same work that $\Sigma$ is a diagonal matrix idenpendent of state $\Sigma_s = \Sigma_{\hat{s}} = \Sigma$, then
    \begin{align*}
        \mathrm{D}_{KL}(\pi(a|s) || \pi(a|\hat{s})) &= \frac{1}{2} (\log |\Sigma_s \Sigma^{-1}_{\hat{s}}| + tr(\Sigma_{\hat{s}}^{-1}\Sigma_s)) + (\mu_{\hat{s}} - \mu_{s})^T \Sigma_{\hat{s}}^{-1}(\mu_{\hat{s}}-\mu_s) - |\bA|)\\
        &= \frac{1}{2} (\log |\Sigma \Sigma^{-1}| + tr(\Sigma^{-1}\Sigma)) + (\mu_{\hat{s}} - \mu_{s})^T \Sigma^{-1}(\mu_{\hat{s}}-\mu_s) - |\bA|)\\
        &\leq \frac{1}{2}(\mu_{\hat{s}} - \mu_{s})^T \Sigma^{-1}(\mu_{\hat{s}}-\mu_s)\\
        &\leq \frac{1}{2}\Norm{\mu_{\hat{s}} - \mu_{s}}_2 \Norm{\Sigma}_2\\
        &\leq kL_\pi^{FC} \Norm{s - \hat{s}}_p \Norm{\Sigma}_2
    \end{align*}
    where $k$ is some normalizing constant depends on the choice of $p$-norm.
\end{proof}

\newpage
\section{Experiment Details}
\subsection{Experiment - Detailed Setup}
We ran all experiments on a computer with an  Intel Xeon 6248R CPU and a NVIDIA A100 GPU.
\subsubsection{BC hyperparameters}
\label{sec: bc setup}
The common hyperparameters for all BC algorithms are listed in Table \ref{tab:hyper parameters for BC}:
\begin{table}[h!]
    \centering
    \begin{tabular}{ll}
        \toprule
        \textbf{Parameter} & \textbf{Value} \\
        \midrule
        Total timesteps & 500000 \\
        %Evaluation frequency & 5000 \\
        Learning rate & 3e-4 \\
        Batch size & 256 \\
        Optimizer & Adam\\
        \bottomrule
    \end{tabular}
    \caption{Hyperparamters}
    \label{tab:hyper parameters for BC}
\end{table}

For training the policy in \textit{Mojuco}, the size of the hidden layer of the neural network is set to be $[512, 512]$ and for training the policy in Cartpole, the hidden layer size is set to be $[64]$.

\subsection{Baseline Algorithm and Noises}
\label{sec: intro to noises}
\subsubsection{Random noise}
The $\infty$-noise sampler first sample a noise from a uniform distribution between $[-1, 1]$ and then scales to the noise level specified. For noise bounded in $2$-norm, samples are first taken from a Gaussian distribution with 0 mean and 1 std and then normalized with respect to its $2$-norm to be the same as the noise level specified. 

\subsubsection{SR$^2$L and Adversarial Noise}
SR$^2$L \cite{shen2020deep} trains a smooth regularizer for the policy network that aims to smooth the output of the policy network, encouraging the generation of similar actions when confronted with perturbed states \cite{shen2020deep}. It begins by defining a perturbation set $\mathcal{B}_d (s, \epsilon) = \left\{\hat{s}: d(\hat{s}, s) \leq \epsilon\right\}$, where $d$ represents an $l_p$ metric. The regularizer is then formulated as follows:
\begin{align}
    R_x^\pi = \E_{s \sim \rho^\pi} \left[\max_{\hat{s} \in \mathcal{B}_d (s, \epsilon)} \mathrm{D} \left(\pi_\theta(\cdot|s), \pi_\theta(\cdot|\hat{s})\right)\right], \label{eq: regularizer}
\end{align}
where $\mathcal{D}$ is a distance between probability distributions. In the case of stochastic action space, we choose $\mathrm{D}$ to be $1$-norm as it is equivalent to $\mathcal{D}_{TV}$ and be $2$-norm in the case of deterministic actions.

To solve the inner maximization effectively, we introduce an adversarial neural network $\hat{s} = \mu_\phi(s): \mathcal{S} \rightarrow \mathcal{S}$ parameterized by $\phi$ that outputs a perturbed state. The objective of this adversarial network is to identify a perturbed state within a predefined perturbation set $\mathcal{B}_d$ that can induce the highest possible deviation in a selected action from the true state. The optimal adversary is found by a gradient ascent at state $x$,
\begin{align}
    \phi_{t+1} &= \phi_t + \eta_1 \frac{\partial \mathcal{D}}{\partial \phi}(\pi_\theta(s), \pi_\theta(\mu_\phi(s))).
\end{align}

This adversarial neural network will be an auxiliary loss for the vanilla BC algorithm. 
\begin{align}
    L_\theta = \lambda \min_\theta \max_{\hat{s} \in \mathcal{B}_d (s, \epsilon)} \mathrm{D} \left(\pi_\theta(\cdot|s), \pi_\theta(\cdot|\mu_\phi(s))\right)
\end{align}

When the actor minimizes with respect to this auxiliary loss, it can output a similar action when a similar state is fed to the network and thus make the policy "smoother". For every gradient step when we train the imitator policy, the adversarial neural work $\mu_\phi$ is trained 5 times with a learning rate $\eta_1 = 0.001$. The $\lambda$, which controls the robustness level of the policy is set to be $0.1$. Also, this adversarial network $\mu_\phi$ is then used as a noise sampler, namely Adversarial Noise, during the evaluation.

\subsubsection{Robust Sarsa Noise}
The Robust Sarsa (RS) attack is a method designed to test the robustness of policies against adversarial perturbations in state observations. It operates independently of specific critic networks, addressing limitations that rely heavily on the quality of the learned Q function. Unlike traditional methods that depend on a specific Q function learned during training, the RS attack does not rely on the quality of this critic network. This independence ensures the attack's effectiveness even when the critic network is poorly learned or obfuscated. During the evaluation, the policy is fixed, allowing the corresponding action-value function to be learned using on-policy temporal-difference (TD) algorithms similar to Sarsa. This learning process does not require access to the original critic network used during training. To ensure that the learned Q function is robust against small perturbations, an additional robustness objective is introduced. This objective penalizes significant changes in the Q function  due to small changes in the action. The RS attack optimizes a combined loss function that includes the TD loss and a robustness penalty. The robustness penalty ensures that the Q function remains stable within a small neighborhood of actions, defined by a set $B(a_{i})$. We use Stochastic Gradient Langevin Dynamics (SGLD) to solve the inner maximization problem when computing the robustness penalty. SGLD is an optimization method that combines stochastic gradient descent (SGD) with Langevin dynamics. It is particularly useful for sampling from a probability distribution in high-dimensional spaces and has applications in training robust neural networks.

The key hyperparameters of our implementation include a learning rate set to 1e-3, ensuring a balanced pace of learning. We utilize 512 steps per environment per update to adequately sample the state space. To enhance training efficiency and robustness, we deploy 4 parallel environments. The model undergoes updates 50 times to refine the policy iteratively. Additionally, the gradient descent process within the Stochastic Gradient Langevin Dynamics is executed 100 times, allowing thorough exploration and optimization of the parameter space.

%\bibliographystyle{aaai2026}
%\bibliography{aaai2026}  % Make sure this matches your .bib file
%\end{document}

\end{document}